\documentclass[11pt]{article}
\usepackage{brainberry}
\newcommand{\Cref}{\pref}
\newcommand{\samp}{\mathsf{Samp}}
\newcommand{\appsamp}{\mathsf{AppSamp}}
\newcommand{\insamp}{\mathsf{InvSamp}}
\newcommand{\appinsamp}{\mathsf{AppInvSamp}}

\newcommand{\epsDet}{\varepsilon_{\mathrm{det}}}

\newcommand{\decode}{\mathsf{Dec}}

\newcommand{\goodsampler}{(B,C,\alpha)}

\newcommand{\idc}{\mathds{1}}

\newcommand{\resolve}{\mathsf{Resolve}}

\newcommand{\iid}{i.i.d.\@ }

\newcommand{\poscon}{\mathsf{Pos}\_\mathsf{Con}}
\newcommand{\posseq}{\mathsf{Pos}\_\mathsf{Seq}}

\newcommand{\pre}{\mathsf{Pre}}

\newboolean{doubleblind}
\setboolean{doubleblind}{false}

\title{\Large Learning $\mathsf{AC}^0$ under Locally Sampleable Graphical Models}
\date{}

\ifthenelse{\boolean{doubleblind}}{\author{Author(s)}}{
\author{Weiming Feng\thanks{School of Computing and Data Science, The University of Hong Kong.  Email: \texttt{wfeng@hku.hk}} 
\and  Xiongxin Yang\thanks{Department of Computer Science, University of California, Santa Barbara. Email: \texttt{xiongxinyang@ucsb.edu}}  
\and  Yixiao Yu\thanks{State Key Laboratory for Novel Software Technology, New Cornerstone Science Laboratory, Nanjing University. Emails: 
\texttt{yixiaoyu@smail.nju.edu.cn, zhangyiyao@smail.nju.edu.cn}}  
\and  Yiyao Zhang\footnotemark[3]}
}

\begin{document}
\maketitle
\begin{abstract}

The problem of learning constant-depth circuits holds profound implications for computational learning theory.
In a seminal result, by introducing the low-degree algorithm, Linial, Mansour, and Nisan (J.\ ACM 1993) 
presented a quasipolynomial-time learner for $\mathsf{AC}^0$ under the uniform distribution. 
However, obtaining comparable learning guarantees for broader classes of correlated distributions has remained a longstanding challenge. 
Recently, Chandrasekaran, Gaitonde, Moitra, and Vasilyan (arXiv 2026) extended these guarantees to Gibbs distributions on bounded-degree graphical models with both strong spatial mixing and polynomial growth.

In this paper, we give a quasipolynomial-time learner for $\mathsf{AC}^0$ under graphical models that admit efficient local samplers, circumventing the polynomial-growth requirement in prior work.
The key ingredient is a new low-degree approximation for Gibbs distributions, established by simulating and suitably truncating the classical Glauber dynamics.
As applications, this framework yields learners for two-spin systems, including the hard-core model and Ising model, on arbitrary bounded-degree graphs, in regimes approaching their respective sampling thresholds.

\end{abstract}
\setcounter{tocdepth}{1}
\tableofcontents

\newpage
\section{Introduction}
Learning small-depth Boolean circuits is a classical and fundamental problem in learning theory.  
In a seminal result, Linial, Mansour and Nisan~\cite{LMN93Constant} showed that every polynomial-size $\mathsf{AC}^0$ circuit has a quasipolynomial-degree Fourier approximation under the uniform distribution, 
leading to a quasipolynomial-time learning algorithm. 
The same technique also applies broadly to product distributions since independence gives a Fourier basis, and low-degree approximation turns learning into low-dimensional regression~\cite{FJS91Improved,BOW10Polynomial}.

Many natural distributions, however, are far from product. 
The graphical models is a succient way to represent complex joint distributions, which are widely used in physics, probability, and computer science.
Given a underlying graph $G=(V,E)$, a Gibbs distribution $\mu$ is a probability distribution on the configuration space $\{\pm 1\}^V$, where each vertex $v\in V$ is associated with a Boolean random variable and adjacent variables are coupled through the edges in $E$.
We consider the following problem of learning $\mathsf{AC}^0$ function $f:\{\pm 1\}^V\to\{\pm 1\}$ under graphical models.

\vspace{0.3cm}
\begin{tcolorbox}[
  sharp corners,
  colback=white,
  colframe=Black,
  width=0.98\linewidth,
  center,
  before skip=6pt,
  after skip=6pt,
  top=6pt,
  bottom=4pt
]
{\centering\textbf{Learning $\mathsf{AC}^0$ function $f$ under Gibbs distribution $\mu$}\par}

\begin{description}[leftmargin=!,labelwidth=\widthof{\bfseries Output:}]
  \item[Input:] $N$ samples $(x_1,f(x_1)),\ldots,(x_N,f(x_N))$,
  where $x_i\sim\mu$.
  \item[Output:] A hypothesis $h:\supp(\mu)\to\{\pm 1\}$ such that
  $\Pr[x\sim\mu]{f(x)\neq h(x)} \le \varepsilon$.
\end{description}
\end{tcolorbox}
\vspace{0.3cm}

A recent work of Chandrasekaran, Gaitonde, Moitra and Vasilyan~\cite{CGMV26Learning} made an important step toward this problem: they gave a learning algorithm for $\mathsf{AC}^0$ under bounded-degree Gibbs distributions satisfying strong spatial mixing and polynomial growth. 
Strong spatial mixing is a natural property of Gibbs distributions, which ensures that the correlation between random variables decays with their graph distance.
Polynomial growth requires that for every vertex $v \in V$, the number of vertices within distance $\ell$ of $v$ is at most $\poly(\ell)$. 
Combining these two properties, they showed that every $\mathsf{AC}^0$ function $f$ can be approximated by a low-degree polynomial $g$ under the Gibbs distribution, in the sense that $\E[x\sim\mu]{(f(x)- g(x))^2} \le \varepsilon$, which yields the first quasipolynomial-time learning algorithm for certain Gibbs distributions.


The result in \cite{CGMV26Learning} opened up the study of learning $\mathsf{AC}^0$ under genuinely dependent Gibbs distributions. The polynomial-growth assumption captures many important geometric graphs, such as the lattices $\mathbb Z^d$. A natural next step is to go beyond this geometric setting: on a general bounded-degree graph, a ball of radius $\ell$ can contain $\exp(\Omega(\ell))$ vertices, as in expander graphs and Erd\H{o}s-R\'enyi random graphs. In this work, we address this direction by proving learning guarantees for locally sampleable graphical models on general bounded-degree graphs.

\subsection{Our Results}

Our main conceptual result is a new connection: if the Gibbs distribution admits a good local sampler, then every $\mathsf{AC}^0$ function $f$ can be approximated by a low-degree polynomial $g$ under the same distribution. 
More precisely, we consider a fundamental sampling algorithm for Gibbs distributions: the systematic-scan Glauber dynamics. It is a Markov chain that starts from a fixed configuration $X = \sigma \in \{\pm 1\}^V$ and updates the vertices in a fixed cyclic order $u_1,u_2,\ldots,u_T$, where $u_t$ is the $(t \bmod n)$-th vertex of the graph. In the $t$-th step, the chain resamples the spin at vertex $u_t$ conditioned on the current spins of the neighbors of $u_t$. Let $r_t$ denote the randomness used in the $t$-th resampling step, and let $\*r = (r_1,r_2,\ldots,r_T)$ collect all the randomness used by the chain.
Suppose the chain mixes after $T = O(n \log n)$ steps, so that the final configuration $X_T$ is distributed (approximately) as $\mu$. Roughly speaking, a good local sampler is a procedure that, for any requested vertex $v$, outputs the spin $X_T(v) \in \{\pm 1\}$ by querying only $O(\log n)$ coordinates of the random string $\*r$. Moreover, the queries are made in a local and adaptive fashion: the procedure reads a constant number of coordinates of $\*r$ at a time, and the values it observes determine which coordinates to read next.
Similar local marginal samplers have been studied in the literature on Local Computation Algorithms (LCAs) for sampling~\cite{BiswasRY20, FGW25Derandomizing, DongM25}. Our result further connects them to the problem of learning Boolean functions.


%


Compared to the strong spatial mixing property used in~\cite{CGMV26Learning}, the local sampler provides a more refined control on the dependence of the random variables. For many natural graphical models, such a good local sampler exists even if the graph does not have polynomial growth.

We now give our learning results for different graphical models.
For any Boolean function $f: \set{\pm 1}^n \to \set{\pm 1}$, we say $f \in \mathsf{AC}(d,n^c)$ if it can be computed by a depth-$d$ circuit with $n^c$ gates, where every gate has unbounded fan-in.
We first state the result for the hard-core model.
For a graph $G=(V,E)$, the hard-core distribution with fugacity $\lambda>0$ is the distribution on independent sets $I\subseteq V$ with probability proportional to $\lambda^{|I|}$.
Equivalently, in the two-spin notation used later, spin $+1$ denotes occupation and adjacent vertices cannot both have spin $+1$.

\begin{theorem}[Learning under the hard-core model]
  \label{thm:intro-hardcore-learning}
  Fix constants $c,\Delta$ and $\eta\in(0,1)$, with $\Delta\ge2$.
  Let $G$ be an $n$-vertex graph of maximum degree at most $\Delta$, and let $\mu$ be the hard-core distribution on $G$ with fugacity $\lambda < (1-\eta)/(\Delta-1)$.
  Then for every $\varepsilon\in(0,1)$, there is an algorithm that, given $N=n^{\log^{O(d)}(n/\varepsilon)}$ \iid samples from $\mu$ labeled by an unknown $f\in\mathsf{AC}(d,n^c)$, runs in time $n^{\log^{O(d)}(n/\varepsilon)}$ and with probability at least $0.9$ outputs a hypothesis $h:\supp(\mu)\to\{\pm 1\}$ satisfying $\oPr_{\sigma\sim\mu}[h(\sigma)\neq f(\sigma)]\le\varepsilon$,
  where the hidden constant in $O(d)$ depends only on $c, \Delta, \eta$.
\end{theorem}

In the above theorem, the condition $\lambda < \frac{1-\eta}{\Delta - 1} \approx \frac{1}{\Delta}$ matches the classical Dobrushin threshold for the hard-core model. In comparison, the strong spatial mixing property assumed in~\cite{CGMV26Learning} holds for the hard-core model when $\lambda < \frac{(\Delta-1)^{\Delta - 1}}{(\Delta-2)^\Delta} \approx \frac{e}{\Delta}$~\cite{Weitz06}. Hence, our result requires a condition on the fugacity that is stronger only by a constant factor, while the underlying graph is no longer required to have polynomial growth.

The hardcore model poses a hard constraint that two adjacent vertices cannot both have spin $+1$. We next consider graphical models with more flexible soft constraints. Given a graph $G= (V,E)$, suppose each edge $e$ has a symmetric interaction matrix $A_e:\{\pm 1\}^2\to\mathbb R_{> 0}$ and each vertex $v$ has an external field $h_v:\{\pm 1\}\to\mathbb R_{> 0}$. The Gibbs distribution is given by
\begin{equation*}
\forall \sigma \in \{\pm 1\}^V, \quad \mu(\sigma) \propto \prod_{e\in E} A_e(\sigma_e) \prod_{v\in V} h_v(\sigma_v),
\end{equation*}
We impose the usual normalized soft-constraint condition, as in~\cite{LWY26Local}.
Namely, write $\widetilde A_e=A_e/\max_{\chi_1, \chi_2 \in \{\pm 1\}} A_e(\chi_1,\chi_2)$. Our learning result for this model is as follows.

\begin{theorem}[Learning under soft constraints]
  \label{thm:intro-soft-learning}
  Fix constants $c,\Delta$ and $\eta\in(0,1)$, with $\Delta\ge1$.
  Let $G=(V,E)$ be an $n$-vertex graph of maximum degree at most $\Delta$, and let $\mu$ be a two-spin Gibbs distribution on $G$ with arbitrary external fields $h_v$ and edge interaction matrices $A_e$.
  Assume that, for every edge $e$ and every $\chi_1,\chi_2\in\{\pm 1\}$,
  \begin{equation*}
    \widetilde A_e(\chi_1,\chi_2)\ge 1-\frac{1-\eta}{2\Delta}.
  \end{equation*}
  Then for every $\varepsilon\in(0,1)$, there is an algorithm that, given $N=n^{\log^{O(d)}(n/\varepsilon)}$ \iid samples from $\mu$ labeled by an unknown $f\in\mathsf{AC}(d,n^c)$, runs in time $n^{\log^{O(d)}(n/\varepsilon)}$ and with probability at least $0.9$ outputs a hypothesis $h:\supp(\mu)\to\{\pm 1\}$ satisfying $\oPr_{\sigma\sim\mu}[h(\sigma)\neq f(\sigma)]\le\varepsilon$,
  where the hidden constant in $O(d)$ depends only on $c, \Delta, \eta$.
\end{theorem}

As a special case, if for all edges $e \in E$, the interaction matrix satisfies $A_e(-1,-1) = A_e(+1,+1) = \beta$ and $A_e(-1,+1) = A_e(+1,-1) = 1$, then the graphical model is the Ising model. The following corollary is immediate from the soft-constraint theorem, since after edge normalization the smallest entry of the Ising interaction matrix is $\min\{\beta,\beta^{-1}\}$.

\begin{corollary}[Learning under the Ising model]
  \label{cor:intro-ising-learning}
  Fix constants $c,\Delta$ and $\eta\in(0,1)$, with $\Delta\ge1$.
  Let $G$ be an $n$-vertex graph of maximum degree at most $\Delta$, and let $\mu$ be the Ising distribution on $G$ with arbitrary external fields and edge activity $\beta>0$.
  If
  \begin{equation*}
    \beta\in\left(1-\frac{1-\eta}{2\Delta},\left(1-\frac{1-\eta}{2\Delta}\right)^{-1}\right),
  \end{equation*}
  then for every $\varepsilon\in(0,1)$, there is an algorithm that, given $N=n^{\log^{O(d)}(n/\varepsilon)}$ \iid samples from $\mu$ labeled by an unknown $f\in\mathsf{AC}(d,n^c)$, runs in time $n^{\log^{O(d)}(n/\varepsilon)}$ and with probability at least $0.9$ outputs a hypothesis $h:\supp(\mu)\to\{\pm 1\}$ satisfying $\oPr_{\sigma\sim\mu}[h(\sigma)\neq f(\sigma)]\le\varepsilon$,
  where the hidden constant in $O(d)$ depends only on $c, \Delta, \eta$.
\end{corollary}

For the Ising model, our condition is roughly $1 - \frac{1}{2\Delta}<\beta < 1+\frac{1}{2\Delta-1}$, while the strong spatial mixing property holds when $1 - \frac{2}{\Delta}< \beta < 1 + \frac{2}{\Delta-2}$. Our condition is stronger, but both lie in the same regime $|\beta - 1| = O(\frac{1}{\Delta})$. In summary, all of our learning results hold on general bounded-degree graphs, at the cost of only a constant factor in the parameters of the graphical models.



\subsection{Related Work}
\paragraph{Learning via Fourier and low-degree methods}
The Fourier-analytic approach to learning Boolean functions via low-degree algorithms originates with the work of Linial, Mansour, and Nisan~\cite{LMN93Constant}. 
Their theorem shows that functions computed by $\mathsf{AC}^0$ circuits admit low-degree $L^2$ approximations under the uniform distribution, which in turn yields quasipolynomial-time learning algorithms. 
Subsequent work has sharpened Fourier-tail bounds for $\mathsf{AC}^0$~\cite{Bop97Average,Haas01Slight,Tal17Tight} and extended this framework to general product distributions~\cite{FJS91Improved,BOW10Polynomial}. 
These developments are closely related to the broader analysis of functions on product spaces~\cite{MOO10Noise,Kel11Influences,OD14Analysis,GJN16Influence}. 
Moreover, the low-degree paradigm has been effective beyond $\mathsf{AC}^0$, including for decision trees~\cite{KM93LearningDecisionTrees,GKK08Agnostically}, DNF formulas~\cite{Jackson97DNF,Fel12Learning}, and geometric concept classes such as halfspaces and intersections of halfspaces~\cite{KlivansOS04Learning,KalaiKKMS08Agnostically}.

\paragraph{Learning under Gibbs distributions}

Although a substantial body of work has sought to extend learning guarantees beyond product distributions,
most notably to smoothed product distributions~\cite{KST09Learning,BDM20ID3,CKK24Smoothed},
learning under more correlated distributions such as Gibbs distributions remains largely unexplored.

Prior to the recent work of~\cite{CGMV26Learning}, 
a notable approach in this direction was the MCMC-learning framework introduced by Kanade and Mossel~\cite{KanadeM15MCMC}. 
Their framework replaces the standard Fourier basis with spectral information derived from a Markov chain associated with the target distribution. 
While this suggests a general paradigm for learning under correlated distributions, 
it requires strong algorithmic access to the corresponding basis, 
a condition that is notoriously difficult to verify for concrete models. 
In another recent development, Chandrasekaran and Klivans~\cite{CK25Learning} showed that logarithmic-size juntas can be efficiently PAC-learned given samples from a general Markov random field with smoothed external fields.

We emphasize that this supervised learning problem is different from the extensive body of work on learning Gibbs distributions~\cite{Bre15Efficiently,VMLC16Interaction,KM17Learning,HKM17Information,WSD19Sparse,CK25Learninga,GMM25Bypassing,GMM26Learning,DKY26Estimating}, 
where the objective is to recover the structure and parameters of an unknown graphical model or Markov random field. 
In our setting, the Gibbs distribution serves as the known ambient distribution under which an unknown Boolean function is learned from labeled examples.

\paragraph{Local samplers}
In the sampling literature, ``local samplers'' generally fall into two categories. The first consists of Local Computation Algorithms (LCAs)~\cite{BiswasRY20, FGW25Derandomizing, DongM25}, 
which require locally computed values at specific sites to be consistent with a global configuration sampled from the target distribution. 
Our local samplers belong to this category.

A closely related concept is the marginal sampler~\cite{AJ22Perfect,AGPP23Perfect,AFG25Sink,LWY26Local}, 
which samples requested marginals without generating the full configuration. 
These have been widely used to obtain fast counting and sampling algorithms for spin systems and constraint satisfaction problems~\cite{AFFGW24Approximate,HWY23Deterministic,CCLZ26Subquadratic}.

\section{Technical Overview}

Let $\mu$ be a Gibbs distribution on $\{\pm 1\}^V$ associated with a graph
$G=(V,E)$, and let $f:\{\pm 1\}^V\to\{\pm 1\}$ be computed by an
$\mathsf{AC}(d,n^c)$ circuit for fixed constants $d,c>0$.
The key of the learning result is to construct a low-degree
polynomial $p:\{\pm 1\}^V\to\mathbb R$ that approximates $f$ in $L^2(\mu)$:
\[
  \E[y\sim\mu]{(f(y)-p(y))^2}\le \eps,
  \qquad
  \deg(p)\le \mathrm{polylog}(n /\eps).
\]
Then the desired quasipolynomial-time
learner follows from the standard result in \cite{LMN93Constant}.

\subsection{The $\samp$--$\insamp$ Framework}

Chandrasekaran, Gaitonde, Moitra, and Vasilyan~\cite{CGMV26Learning} obtain
such low-degree approximations under Gibbs measures through an abstract
$\samp$--$\insamp$ framework. Their technique leads to learning results for Gibbs
distributions on \emph{polynomial-growth} graphs satisfying \emph{strong spatial mixing}.
We first recall this mechanism, and then explain how our approach bypasses the polynomial-growth assumption.

In the ideal form of the framework, $\samp$ is a sampler: it maps a string of
random bits to a configuration with law $\mu$. If $U_T$ denotes the uniform
distribution on $\{0,1\}^T$, then $\samp(x)\sim\mu$ for $x\sim U_T$, and hence
$f(\samp(x))$ has the same law as $f(\sigma)$ for $\sigma\sim\mu$.
If, in addition, the composition $f\circ\samp$ has a low-depth
$\mathsf{AC}^0$ representation over the random bits, then Fourier truncation
gives a low-degree polynomial $g:\{0,1\}^T\to\mathbb R$ such that
\begin{equation}\label{eq:techoverview-appsamp-fourier}
  \E[x\sim U_T]{(f(\samp(x))-g(x))^2}\le \eps.
\end{equation}

The inverse sampler $\insamp$ transfers this approximation back from the
random-bit space to the Gibbs configuration space. It is required to satisfy
two properties:
\begin{enumerate}
    \item \emph{Conditional inversion:} given any configuration
    $\sigma\in\supp(\mu)$, $\insamp(\sigma)$
    returns a uniform random bit string $x\in\{0,1\}^T$ conditioned on
    $\samp(x)=\sigma$;
    \label{item:CGMV-uniformly-revert}
    \item \emph{Locality:} each coordinate of $\insamp(\sigma)$ depends on only $\polylog(n)$
    coordinates of $\sigma$.
    \label{item:CGMV-low-degree}
\end{enumerate}
The first property allows us to replace the random bits $x\sim U_T$ by
$\insamp(\sigma)$ with $\sigma\sim\mu$:
\begin{equation*}
  \text{LHS of}~\eqref{eq:techoverview-appsamp-fourier}
  =
  \E[\sigma\sim\mu]{(f(\samp(\insamp(\sigma)))-g(\insamp(\sigma)))^2}
  =
  \E[\sigma\sim\mu]{(f(\sigma)-g(\insamp(\sigma)))^2}.
\end{equation*}
The second property ensures that the composition $g\circ\insamp$ remains low
degree.\footnote{One might notice that $\insamp$ is a randomized function. Hence, $g \circ \insamp$ is also randomized. However, one can make it deterministic by fixing the internal randomness of $\insamp$ by the probabilistic method.}

In~\cite{CGMV26Learning}, the $\samp$--$\insamp$ pair is constructed from
strong spatial mixing on polynomial-growth graphs. Strong spatial mixing ensures that when sample a random spin at $v \in V$, one only needs to care about the correlation between $v$ and other vertices $u$ within a ball of radius $\polylog(n)$ centered at $v$.
The polynomial-growth assumption is then used to
ensure that this ball contains only $\polylog(n)$ vertices.
This size bound is essential in both directions: with a careful design of the
sampling and inversion procedures, it gives a low-depth implementation of
$\samp$ and guarantees that every coordinate of $\insamp$ has low degree.
Without polynomial growth, a radius-$\polylog(n)$ ball may contain
super-polylogarithmically many vertices, so the same construction no longer
gives a useful degree bound.


\subsection{Our $\samp$--$\insamp$ Pair via Systematic-Scan Glauber Dynamics}

We bypass the polynomial-growth obstruction by using a different sampling
procedure. We take $\samp$ to be a $T$-step systematic-scan Glauber dynamics
along a fixed scan sequence $u_1,\ldots,u_T\in V$.
As a running example, consider the hard-core model on $G$ with fugacity
$\lambda>0$.
Starting from a fixed initial configuration
$X_0=\sigma_0\in\supp(\mu)$, the Glauber dynamics updates the spin at $u_t$
at step $t$ according to the conditional law
$\mu_{u_t}^{X_{t-1}(N(u_t))}$, namely the marginal law of $\mu$ at $u_t$
conditioned on the current spins of its neighbors. We implement this heat-bath
update using an auxiliary mark:
\begin{enumerate}
  \item Draw a mark $\rho_t\sim\+P_t$. For the hard-core model,
  $\+P_t$ is supported on $\{0,\bot\}$, with
  $\Pr{\rho_t=0}=1/(1+\lambda)$ and
  $\Pr{\rho_t=\bot}=\lambda/(1+\lambda)$.
  \item Update $X_t(u_t)$ to
  $\psi_t(\rho_t,X_{t-1}(N(u_t)))$, where $\psi_t$ is deterministic and is
  chosen so that this output has law $\mu_{u_t}^{X_{t-1}(N(u_t))}$ when
  $\rho_t\sim\+P_t$. For the hard-core model, the mark $0$ forces the output
  to be $-1$, while the mark $\bot$ attempts to set the spin to $+1$ and
  succeeds exactly when all neighbors currently have spin $-1$; otherwise the
  output is $-1$.
\end{enumerate}
Thus $\samp$ is the deterministic map that sends the mark sequence
$\*\rho\sim\+P_{1:T}\defeq\bigotimes_{t=1}^T\+P_t$ to the terminal
configuration $X_T$ of this Glauber trajectory. Since the initial state is
fixed, this is not an exact sample from $\mu$ for finite $T$. We will handle this error in the technical sections.

The inverse sampler $\insamp$ is defined using reversibility of the single-site
Glauber updates. To illustrate the idea, first consider the \emph{idealized} setting
where the initial configuration is drawn from stationary distribution
$X_0\sim\mu$. Let $(X_t)_{t=0}^T$ be the resulting trajectory, and let $P_t$
be the heat-bath transition kernel at site $u_t$. Reversibility gives the path
identity
\begin{align}\label{eq:techoverview-idealized-reversibility}
  \mu(X_0)\prod_{t=1}^T P_t(X_{t-1},X_t)
  =
  \mu(X_T)\prod_{t=1}^T P_t(X_t,X_{t-1}).
\end{align}
Inspired by the above identity, given the terminal configuration $X_T$, our $\insamp$ procedure is defined as follows: we first sample a
backward trajectory $(X_{T-1},\ldots,X_0)$ by running the same heat-bath
updates in the reverse scan order $u_T,u_{T-1},\ldots,u_1$. Once this
trajectory is sampled, the marks are recovered independently across times by
sampling each $\rho_t$ from $\+P_t$ conditioned on being consistent with the
observed transition: $\psi_t(\rho_t,X_{t-1}(N(u_t)))=X_t(u_t)$.



\subsection{Determining Mark Sequence and Local Sampler}

The reversibility argument above is most transparent when the forward chain is
initialized from stationarity. 
In the actual construction, however, $\samp$
starts from a fixed configuration $X_0=\sigma_0\in\supp(\mu)$ so that it is a
deterministic function of the mark sequence. To bridge this gap, we introduce
the notion of a \emph{determining mark sequence}, in the spirit of \emph{coupling
from/towards the past} ({CFTP}/{CTTP})~\cite{ProppW96,FGW25Derandomizing}.
By the definition of systematic-scan Glauber dynamics, the terminal configuration $X_T$ is a function of the mark sequence $\*\rho$ and the initial configuration $\sigma_0$. To emphasize the dependence on $\sigma_0$, we write our $\samp(\*\rho)$ as $\samp_{\sigma_0}(\*\rho)$. 
We call a mark sequence $\*\rho$ determining if the terminal configuration
$X_T = \samp(\*\rho)$ produced by the systematic-scan Glauber dynamics is independent of
the initial configuration $X_0$:
\begin{align*}
\forall \sigma_0,\sigma_0'\in\supp(\mu),\quad \samp_{\sigma_0}(\*\rho) = \samp_{\sigma_0'}(\*\rho).
\end{align*}
On this event, starting from the fixed
configuration $X_0=\sigma_0$ gives the same terminal state as starting from
$X_0\sim\mu$, so the stationary intuition behind the $\samp$--$\insamp$
comparison applies.

To show that determining mark sequences occur with high probability, we use
our second tool: a \emph{local sampler} for the systematic-scan Glauber
dynamics. The Glauber dynamics generates the full configuration
$X_T\in\{\pm 1\}^V$, whereas a local sampler resolves only one requested spin
$X_T(v)$.

Again consider the hard-core model. To resolve $X_T(v)$, look back to the
last time $t\le T$ at which $v$ was updated. If the mark at that time is
$\rho_t=0$, then the update rule forces $X_T(v)=X_t(v)=-1$, so the value is
determined immediately by the mark. If instead $\rho_t=\bot$, the update
attempts to set $v$ to $+1$, and this succeeds only when all neighbors of
$v$ were unoccupied just before time $t$. Thus the resolver recursively asks
for the relevant neighbor values at their latest update times before $t$.
Each recursive branch stops either when it encounters a forcing mark
$\rho_s=0$, or when it reaches time $0$, in which case it reads the initial
configuration $X_0$.

The point is that, when $\rho_t = 0$ has sufficiently large probability,
these recursive explorations typically stop quickly. If the resolver for
every vertex $v \in V$ determines $X_T(v)$ without ever reading $X_0$, then the
entire terminal configuration $X_T$ is independent of the initial configuration,
and the mark sequence is determining. 
For all other spin systems considered in this
paper, by carefully designing the mark distribution $\+P_t$ and deterministic update rule $\psi_t$, we can show that with high probability over $\*\rho \sim \+P_{1:T}$, the local sampler terminates quickly, and hence $\*\rho$ is determining.

As a consequence, we can compare the following two joint processes for generating a pair $(\*\rho,\sigma)$: the \emph{forward} process samples $\*\rho\sim\+P_{1:T}$ and runs the sampler to obtain $\sigma=\samp(\*\rho)$, while the \emph{backward} process samples $\sigma\sim\mu$ and recovers $\*\rho=\insamp(\sigma)$. Whenever the mark sequence is determining, the fixed initial configuration behaves exactly as a stationary one, and the reversibility argument shows that the two processes assign the same probability to the pair $(\*\rho,\sigma)$. Since the mark sequence is determining with high probability, the two joint distributions are close in total variation distance. This gives an approximate version of the conditional inversion property in~\pref{item:CGMV-uniformly-revert}, which suffices for our purpose. See \pref{lem:sample-inv-comparison} for the formal statement.



\subsection{Approximating the $\samp$--$\insamp$ Pair via Truncation}

It remains to control the degree of the approximating polynomial. Since the
local resolver terminates quickly with high probability, we truncate its
execution after $\polylog(n / \eps)$ steps. Applying this truncated resolver to
each vertex gives an approximate sampler $\appsamp$. The truncation error is
small, so it suffices to approximate $f\circ\appsamp$ on the product mark
space. Concretely, we seek a low-degree polynomial
$g:\supp(\+P_{1:T})\to\mathbb R$ such that
\begin{equation}
    \label{eq:overview-appsamp-fourier}
    \E[\*\rho\sim\+P_{1:T}]{\abs{f\circ\appsamp(\*\rho)-g(\*\rho)}^2}\le\varepsilon.
\end{equation}
The reason this is possible is that $\appsamp$ has low circuit complexity.
Let $\sigma=\appsamp(\*\rho)$. For each vertex $v$, the value $\sigma(v)$ is
computed by running the local sampler only up to the truncation depth
$\polylog(n / \eps)$. During this truncated computation, the sampler queries only
a small number of marks, and each recursive step can branch only through the
constant-size neighborhood of the current vertex. Thus, on bounded-degree
graphs, the whole truncated recursion has only polynomially many possible
query patterns. Expanding over these possible query patterns gives a
constant-depth polynomial-size circuit for each coordinate $\sigma(v)$, whose
inputs are indicator variables of the form $\idc[\rho_t=c]$, where $c \in \supp(\+P_t)$.
Since $f$ itself is computed by an $\mathsf{AC}^0$ circuit, the composition
$f\circ\appsamp$ is again computed by a constant-depth polynomial-size circuit such that all the input variables are indicator Boolean variables in the form of $\idc[\rho_t=c]$.

There is one technical point: the mark variables need not be Boolean. In the
hard-core example, each $\+P_t$ is supported on the two-symbol alphabet
$\{0,\bot\}$, so the mark space can be identified with a Boolean cube. For
more general spin systems, the marks take values in a finite alphabet $[Q]$ for some $Q>2$.
Thus $f\circ\appsamp$ is a Boolean-valued function on the product space
$[Q]^T$ equipped with the product distribution
$\+P_{1:T}=\bigotimes_{t=1}^T\+P_t$. We therefore need a Fourier
approximation theorem for $\mathsf{AC}^0$ circuits over finite product
domains. This is only a technical extension of the Boolean low-degree
approximation statement: we use the product-space Efron--Stein decomposition
and measure degree by the number of original mark coordinates involved. The
proof uses a random two-point selector reduction, which turns each fixed fiber
into an ordinary Boolean $\mathsf{AC}^0$ circuit, applies the Boolean Fourier
tail bound there, and then averages back to the product space. In
\pref{sec:low-degree-product}, this gives a low-degree approximating
polynomial $g:[Q]^T\to\mathbb R$ for $f\circ\appsamp$.

Finally, the inverse sampler $\insamp$ has the same local structure, but run
in the reverse scan order. Applying the same truncation idea gives
$\appinsamp$, an approximate inverse sampler whose output coordinates depend
on only polylogarithmically many coordinates of the input configuration. This
is the locality property in~\pref{item:CGMV-low-degree} needed to keep
$g\circ\appinsamp$ low degree.

\section{Preliminaries}
\subsection{Graphical Models}

Let $G=(V,E)$ be a finite graph with $n=|V|$.  For $v\in V$, write $N(v)=\{u\in V:\{u,v\}\in E\}$ for the neighborhood of $v$, and let $\Delta=\max_{v\in V}|N(v)|$.

\begin{definition}[Graphical model]
  Let $\Omega$ be a finite spin set and let $\mathcal X=\Omega^V$ be the configuration space.
  A graphical model on $G$ is specified by nonnegative vertex potentials $\lambda_v:\Omega\to\mathbb R_{\ge 0}$ for $v\in V$, and nonnegative edge potentials $A_e:\Omega\times\Omega\to\mathbb R_{\ge 0}$ for $e\in E$.
  We take each $A_e$ to be symmetric in its two arguments.
  It defines the \emph{Gibbs distribution} $\mu$ on $\mathcal X$ by
  \begin{equation*}
    \mu(\sigma) =\frac{1}{Z} \prod_{v\in V}\lambda_v(\sigma(v))
    \prod_{\{u,v\}\in E}A_{\{u,v\}}(\sigma(u),\sigma(v)),
  \end{equation*}
  where $Z$ is the normalizing constant, also known as the partition function.
  To avoid trivialities, we assume that $\mu$ is well defined, i.e., $Z>0$.
\end{definition}

In the two-spin setting used throughout the applications, the spin set is $\Omega=\{\pm 1\}$, and the model is specified by the local weights $\lambda_v(-1),\lambda_v(+1)$ and the edge interaction matrices $A_e(s,t)$, $s,t\in\{\pm 1\}$.
Hard constraints are allowed by letting some entries of $A_e$ be zero.
For example, the hard-core model on $G$ with fugacity $\lambda>0$ is given by $\Omega=\{\pm 1\}$, $\lambda_v(-1)=1$, $\lambda_v(+1)=\lambda$ for all $v\in V$, and $A_{\{u,v\}}(s,t)=\idc[s=-1\text{ or }t=-1]$ for all $\{u,v\}\in E$ and $s,t\in\{\pm 1\}$.

Under the Gibbs distribution $\mu$ defined above, for a configuration $\tau\in\supp(\mu)$, a vertex $v\in V$, and a spin $s\in\Omega$, the probability of $v$ taking spin $s$ conditioned on the neighboring spins in $\tau$ is given by 
$\mu_v^{\tau(N(v))}(s) = \Pr[X \sim \mu]{X(v) = s \mid \, X(N(v)) = \tau(N(v))}$.
In terms of the potentials above,
\begin{equation*}
  \mu_v^{\tau(N(v))}(s) = \frac{\lambda_v(s)\prod_{u\in N(v)} A_{\{u,v\}}(\tau(u),s)}{\sum_{a\in\Omega} \lambda_v(a)\prod_{u\in N(v)} A_{\{u,v\}}(\tau(u),a)},
\end{equation*}
whenever the denominator is positive.

\subsection{Systematic-Scan Glauber Dynamics}
For a configuration $\sigma\in\supp(\mu)$ and a subset $S\subseteq V$, write $\sigma(S)$ for the restriction of $\sigma$ to $S$.
For each vertex $v\in V$, define the fixed-site heat-bath kernel $P_v$ on $\set{\pm 1}^V$ by
\begin{equation}
  \label{eq:fixed-site-heat-bath-kernel}
  \forall \sigma, \sigma' \in \supp(\mu), \quad P_v(\sigma,\sigma') = \idc[\sigma'(V\setminus\{v\})=\sigma(V\setminus\{v\})] \cdot \mu_v^{\sigma(N(v))}(\sigma'(v)).
\end{equation}
Observe that $P_v$ leaves all coordinates except $v$ unchanged and resamples the spin at $v$ from the one-site conditional law $\mu_v^{\sigma(N(v))}$.

We use the deterministic cyclic scan associated with the ordering
$V=\{v_1,\ldots,v_n\}$.  For $t\ge 1$, define $i_t\defeq((t-1)\bmod n)+1$
and let $u_t\defeq v_{i_t}$.
Hence the first update is at $v_1$, the second at $v_2$, and after $v_n$ the scan returns to $v_1$.
Given an initial configuration $\sigma_0\in\supp(\mu)$ and a running time $T$, the length-$T$ systematic-scan Glauber trajectory $\sigma_0,\sigma_1,\ldots,\sigma_T$ is generated as follows.  
At each time $t=1,\ldots,T$, only the vertex $u_t$ is updated: for every $w\ne u_t$, set $\sigma_t(w)=\sigma_{t-1}(w)$, and sample $\sigma_t(u_t)\sim \mu_{u_t}^{\sigma_{t-1}(N(u_t))}$.
Equivalently, conditional on $\sigma_{t-1}$, the transition from
$\sigma_{t-1}$ to $\sigma_t$ is given by $P_{u_t}$.

Each fixed-site heat-bath kernel $P_v$ is reversible with respect to the Gibbs
distribution $\mu$. Hence, for the deterministic update sequence $u_1,\ldots,u_T$,
it holds that $\mu P_{u_1}P_{u_2}\cdots P_{u_T}=\mu$.
The resulting scan kernel $\prod_{t=1}^n P_{u_t}$ need not be reversible, but its stationary paths can still be reversed site by site.  
Indeed, if $\sigma_0\sim\mu$ and the updates $P_{u_1},\ldots,P_{u_T}$ are applied in order, then
\begin{align}
  \label{eq:stationary-path-reversal}
  \mu(\sigma_0)\prod_{t=1}^T P_{u_t}(\sigma_{t-1},\sigma_t) &= \mu(\sigma_0)P_{u_1}(\sigma_0,\sigma_1)P_{u_2}(\sigma_1,\sigma_2)\cdots P_{u_T}(\sigma_{T-1},\sigma_T)\notag\\
   &= \mu(\sigma_T)P_{u_T}(\sigma_T,\sigma_{T-1})\cdots P_{u_1}(\sigma_1,\sigma_0) = \mu(\sigma_T)\prod_{t=1}^T P_{u_t}(\sigma_t,\sigma_{t-1}).
\end{align}

\subsection{Low-Degree Approximation of \texorpdfstring{$\mathsf{AC}(d, s)$}{AC(d,s)} Circuits and Learning Algorithm}

\label{sec:low-degree-learning}

In the standard Boolean model, $\mathsf{AC}(d, s)$ denotes the class of functions computed by depth-$d$ circuits of size at most $s$ with unbounded fan-in AND and OR gates and input literals.
Here depth is the length of the longest path from an input gate to the output gate, and size is the total number of gates.
The circuit gates take values in $\{0,1\}$; when the computed function is viewed as $\{\pm 1\}$-valued, we identify output $1$ with $+1$ and output $0$ with $-1$.

We recall the Boolean Fourier notation and the Fourier tail estimate of Tal~\cite{Tal17Tight}; for background on Fourier analysis over Boolean and finite product spaces, see~\cite[Chapters~1 and~8]{OD14Analysis}.
For a positive integer $n$ and $U \subseteq [n]$, let $\chi_U(y) = \prod_{i \in U} y_i$ be the corresponding character on $\{\pm 1\}^n$, with $\chi_{\emptyset} \equiv 1$.
The characters form an orthonormal basis of $L^2(\{\pm 1\}^n)$ under the uniform measure.
Thus every function $f : \{\pm 1\}^n \to \mathbb{R}$ has the expansion
\begin{equation*}
  f(y) = \sum_{U \subseteq [n]} \widehat{f}(U)\chi_U(y), \qquad \widehat{f}(U) = \oE_{y \sim \{\pm 1\}^n}[f(y)\chi_U(y)],
\end{equation*}
where $y \sim \{\pm 1\}^n$ denotes the uniform distribution on $\{\pm 1\}^n$ when the expectation is taken.
The \emph{real multilinear degree} of $f$ is defined as
\begin{equation}
    \deg(f)\defeq\max\set{\abs{U}:\widehat f(U)\neq 0}
\end{equation}
For $r \ge 0$, define the Boolean Fourier tail by
$\mathrm{W}^{\ge r}(f) = \sum_{\abs{U} \ge r}\widehat{f}(U)^2$.
We use the analogous notations $\mathrm{W}^{>r}$, $\mathrm{W}^{=r}$, and write $f^{\le r} = \sum_{\abs{U} \le r}\widehat{f}(U)\chi_U$ for the degree-$r$ Fourier truncation.

We state Tal's Fourier tail bound as follows.
\begin{theorem}[{\cite[Theorem~18]{Tal17Tight}}]
  \label{thm:tal-fourier-tail}
  If a Boolean function $f : \{\pm 1\}^n \to \{\pm 1\}$ is computed by a Boolean $\mathsf{AC}(d, s)$ circuit, then for every integer $r \ge 0$ and let $C_d = 120d \cdot 192^{d - 1}$, it holds that
  \begin{equation*}
  \mathrm{W}^{\ge r}(f) \le 2 \cdot 2^{-r/(C_d\log^{d - 1}(2s))}.
  \end{equation*}
\end{theorem}

This tail estimate immediately yields the standard low-degree approximation over the uniform Boolean cube.
Indeed, let $R$ be a nonnegative integer and let $f^{\le R}$ be the degree-$R$ Fourier truncation of a Boolean $\mathsf{AC}(d, s)$ circuit $f$.
If $R \ge C_d\log^{d - 1}(2s)\log(2/\eps)$, then orthogonality of the Boolean characters and \pref{thm:tal-fourier-tail} give
\begin{equation*}
\oE_{y \sim \{\pm 1\}^n}\left[(f(y)-f^{\le R}(y))^2\right] = \mathrm{W}^{>R}(f) \le \eps.
\end{equation*}
Thus Boolean $\mathsf{AC}(d,s)$ circuits admit $L^2$ approximants of degree $O_d(\log^{d - 1} (s)\log(1/\eps))$ under the uniform distribution.

The above low-degree approximation under the uniform distribution immediately implies efficient learning algorithms for $\mathsf{AC}(d, s)$ under the uniform distribution.
We recall the classic result by Linial, Mansour, and Nisan~\cite{LMN93Constant}.
\begin{theorem}[{~\cite{LMN93Constant}, {~\cite[Theorem 4.5]{CGMV26Learning}}}]\label{thm:linial-mansour-nisan}
  Let $\+D$ be a distribution on $\{\pm 1\}^n$ and $\+F$ be a class of Boolean functions on $\{\pm 1\}^n$. Suppose that for any $f \in \+F$ and $\eps > 0$, there exists a real-valued polynomial $g : \{\pm 1\}^n \to \mathbb{R}$ of degree at most $\ell(\eps)$ such that $\oE_{y \sim \+D}[(f(y) - g(y))^2] \le \eps$.
  Then, there is an algorithm that, given $N = n^{O(\ell(\eps / 2))}$ \iid samples from $\+D$ labeled by an unknown $f \in \+F$, runs in time $\poly(N, n)$ and with probability at least $0.9$ outputs a hypothesis $h : \{\pm 1\}^n \to \{\pm 1\}$ such that $\oPr_{y \sim \+D}[h(y) \neq f(y)] \le \eps$.
\end{theorem}

\section{Low-Degree Approximation under General Product Distributions}
\label{sec:low-degree-product}
In this section, we extend the Boolean low-degree approximation statement from the Boolean cube to finite product domains with arbitrary product distributions.
The argument uses a selector reduction to turn the original circuit into a Boolean $\mathsf{AC}(d,s)$ circuit, and then transfers the Boolean Fourier tail bound back through the product Fourier decomposition.

\paragraph{\texorpdfstring{$\mathsf{AC}(d, s)$ over $[Q]^n$}{AC(d, s) over [Q]n}.}
We extend the $\mathsf{AC}(d, s)$ model to the following \emph{single-coordinate predicate model}.

\begin{definition}[\text{$\mathsf{AC}(d, s)$ over $[Q]^n$}]\label{def:ac-over-q-n}
  Let $n,Q \ge 1$ and $[Q] = \{1,2,\ldots,Q\}$.
  The inputs of the circuit are predicates of the form $\idc[x_i \in A]$, where $i \in [n]$ and $A \subseteq [Q]$, each outputting $1$ if $x_i \in A$ and $0$ otherwise.
  The internal gates are unbounded fan-in AND and OR gates operating on Boolean values in $\{0, 1\}$, and negations are pushed to the inputs.
  A circuit is thus a directed acyclic graph whose sources are predicate gates and whose internal nodes are AND/OR gates.
  Its size is the total number of gates, including both predicate and internal gates, and its depth is the maximum length of an input-output path.

  For $d,s \ge 1$, a function $f : [Q]^n \to \{\pm 1\}$ is in $\mathsf{AC}(d,s)$ in this model if, under the output convention fixed in \pref{sec:low-degree-learning}, it is computed by such a circuit of depth at most $d$ and size at most $s$.
\end{definition}

This differs from the standard Boolean model only in that each input coordinate may take values in a finite alphabet $[Q]$, rather than in $\{0, 1\}$. Predicate gates convert these $[Q]$-valued inputs into Boolean values, while all internal gates are standard unbounded fan-in Boolean gates.

\paragraph{Indicator polynomials on \texorpdfstring{$[Q]^n$}{[Q]n}.}
To extend the notion of ``low-degree'', we use indicator monomials to measure degree on non-Boolean alphabets.
Write $x = (x_1, \dots, x_n) \in [Q]^n$ for the input. A function $f : [Q]^n \to \mathbb{R}$ is an \emph{indicator polynomial} if
\begin{equation}\label{eq:indicator-polynomial}
f(x) = \sum_{S \subseteq [n]} \sum_{a \in [Q]^S} c_{S,a} \prod_{i \in S} \idc[x_i = a_i],
\end{equation}
where $c_{S,a} \in \mathbb{R}$ are real coefficients.
These monomials span all functions $[Q]^n\to\mathbb R$, since the degree-$n$ monomials, with $S = [n]$, are point indicators. 
The representation need not be unique.
As in the Boolean case, define the \emph{coordinate degree} of $f$ as the minimum $d$ such that $f$ has an indicator-polynomial representation using only monomials of support size at most $d$:
\begin{equation}
  \label{eq:coordinate-degree}
  \codeg(f)\defeq\min\left\{d: f(x)=\sum_{\substack{S\subseteq[n]\\ |S|\le d}}\sum_{a\in[Q]^S} c_{S,a}\prod_{i\in S} \idc[x_i = a_i]\text{ for some coefficients }c_{S,a}\in\mathbb R\right\}.
\end{equation}
Equivalently, $\codeg(f)$ is the minimum, over all indicator-polynomial representations of $f$, of the largest coordinate support size. 

With these conventions, the Boolean low-degree approximation theorem has the following product-domain form.
The rest of the section is devoted to its proof.

\begin{theorem}[Low-degree approximation beyond the Boolean domain]
  \label{thm:low-degree-beyond-boolean}
  Let $n,Q,d,s \ge 1$, and let $\+P = \+P^{(1)} \otimes \cdots \otimes \+P^{(n)}$ be a product distribution on $[Q]^n$.
  For any $f: [Q]^n \to \{\pm 1\}$ computed by an $\mathsf{AC}(d, s)$ circuit in the single-coordinate predicate model and any $\eps \in (0, 1)$, set
  \[
  C_d = 120d\cdot 192^{d-1}, \qquad
  T = \min\left\{n,\left\lceil 16C_d\log^{d-1}(2s)\log(2/\eps)\right\rceil\right\}.
  \]
  Then there exists a real-valued indicator polynomial $g : [Q]^n \to \mathbb{R}$ satisfying:
\begin{enumerate}
  \item $\oE_{x \sim \+P}[(f(x) - g(x))^2] \le \eps$;
  \item $\codeg(g) \le T$;
  \item $\norm{g}_{L^\infty(\+P)} \le (2\e n/T)^T$, where $\norm{h}_{L^\infty(\+P)} = \max_{x\in\supp(\+P)}\abs{h(x)}$.
\end{enumerate}
\end{theorem}

\begin{remark}[Relation to the Boolean case]
  When $Q=2$ and $[2]$ is identified with $\{\pm 1\}$, the predicate circuit model is the usual Boolean circuit model and coordinate degree is real multilinear degree.
  Thus \pref{thm:low-degree-beyond-boolean} recovers the Boolean low-degree approximation statement for product measures on the Boolean cube.
\end{remark}

\paragraph{Composition of indicator polynomials.}

We first introduce the one-hot encoding. Consider an arbitrary set $\Omega$ and a map $h:\Omega\to[Q]^n$ from some abstract set $\Omega$ to the product space $[Q]^n$. The following encoding realizes the image space $[Q]^n$ inside the Boolean cube $\set{0,1}^{[n]\times[Q]}$.
\begin{definition}[One-hot encoding]
  \label{def:one-hot-encoding}
  Let $n,Q\ge 1$, let $\Omega$ be a set, and let $h:\Omega\to[Q]^n$.
  The \emph{one-hot encoding} of $h$ is the map
  $h^{\mathsf{oh}}:\Omega\to\set{0,1}^{[n]\times[Q]}$ whose coordinates are
  \begin{equation*}
  \forall i\in[n],\ q\in[Q], \qquad
    h^{\mathsf{oh}}_{i,q}(\omega)
    \defeq
    \idc\interval{h(\omega)_i=q}.
  \end{equation*}
  Thus, for every $\omega\in\Omega$ and $i\in[n]$, exactly one of
  $h^{\mathsf{oh}}_{i,1}(\omega),\ldots,h^{\mathsf{oh}}_{i,Q}(\omega)$ is $1$.

  Conversely, any map $g:\Omega\to\set{0,1}^{[n]\times[Q]}$ with this one-hot property encodes a unique map $\bar g:\Omega\to[Q]^n$ by setting $\bar g(\omega)_i=q$ when $g_{i,q}(\omega)=1$.
\end{definition}

Consider composition with a one-hot-encoded map.
Let $g:\{0,1\}^m\to\{0,1\}^{[n]\times[Q]}$ be a one-hot encoding as in \pref{def:one-hot-encoding}, with coordinate functions $g_{i,q}:\{0,1\}^m\to\{0,1\}$, and let $\bar g:\{0,1\}^m\to[Q]^n$ be the encoded map.
For $f:[Q]^n\to\mathbb R$, we write $f\circ g:\{0,1\}^m\to\mathbb R$ for the function obtained by substituting $g_{i,q}$ for each indicator $\idc[x_i = q]$ in any indicator-polynomial representation of $f(x)$; 
this is well-defined because the one-hot relations make the result equal to $f\circ\bar g$ pointwise.
If $\codeg(f)\le d_2$ and each Boolean function $g_{i,q}$ has real multilinear degree at most $d_1$, then the Boolean function $f\circ g$ has real multilinear degree at most $d_1d_2$.

\paragraph{Product Fourier decomposition on $[Q]^n$.}
Fix a product distribution $\+P = \+P^{(1)} \otimes \cdots \otimes \+P^{(n)}$ on $[Q]^n$.
Let $\Id$ denote the identity operator, i.e., $\Id f = f$. For each $i \in [n]$, let $\-E_i$ be the operator that averages in the $i$-th coordinate with respect to $\+P^{(i)}$, and set $\Delta_i = \Id - \-E_i$.
Thus, for functions $f : [Q]^n \to \mathbb{R}$, the operator acts as
\[
(\-E_i f)(x) = \oE_{y_i \sim \+P^{(i)}}\left[f(x_1, \ldots, x_{i - 1}, y_i, x_{i + 1}, \ldots, x_n)\right].
\]
For any $S \subseteq [n]$, define
\begin{align}\label{eq:product-efron-stein-decomposition}
f_S = \left(\prod_{i \in S} \Delta_i\right)\left(\prod_{j \notin S} \-E_j\right) f, \quad \text{equivalently,} \quad f_S(x) = \sum_{T \subseteq S}(-1)^{|S|-|T|} \oE_{y \sim \+P}[f(x_T,y_{[n] \setminus T})].
\end{align}
The order of the products is irrelevant, since the operators $\-E_i$ and $\Delta_i$ commute across distinct coordinates.
Expanding $\prod_{i = 1}^n (\-E_i + \Delta_i) = \Id$ gives $f = \sum_{S \subseteq [n]} f_S$, the product Fourier decomposition, also known as the Efron--Stein decomposition~\cite[Chapter~8]{OD14Analysis}.
The components are mutually orthogonal in $L^2(\+P)$; in addition, $f_S$ depends only on the coordinates in $S$ and has mean zero in each of these coordinates.

For later use, we introduce the following basis representation of $f_S$.
We use the harmless convention that each marginal $\+P^{(i)}$ has full support; otherwise, one first replaces the $i$-th alphabet by $\supp(\+P^{(i)})$. This does not change any $L^2(\+P)$ error or $L^\infty(\+P)$ norm. Moreover, any approximant on the reduced product domain is a polynomial in the original indicators $\idc[x_i=q]$ for $q\in\supp(\+P^{(i)})$, and the same formula extends to all of $[Q]^n$ without increasing coordinate degree.
For each coordinate $i$, choose an orthonormal basis $\psi_{i,1}, \ldots, \psi_{i,Q}$ of $L^2(\+P^{(i)})$ such that $\psi_{i,1} \equiv 1$ and $\psi_{i,2}, \ldots, \psi_{i,Q}$ span the mean-zero subspace. Concretely, such a basis is obtained by Gram--Schmidt orthonormalization of the indicators $1, \idc[x_i = 2], \ldots, \idc[x_i = Q]$ with respect to the inner product $\langle g, h \rangle_{\+P^{(i)}} = \sum_{q \in [Q]} \+P^{(i)}(q)\, g(q) h(q)$: normalizing the constant gives $\psi_{i,1} \equiv 1$, and orthonormalizing the remaining indicators against it yields the mean-zero functions $\psi_{i,2}, \ldots, \psi_{i,Q}$.
For each $q\in\{2,\ldots,Q\}$, $\psi_{i,q}(x_i)=\sum_{r=1}^Q \alpha^{(i)}_{q,r}\idc[x_i=r]$, with coefficients determined by $\+P^{(i)}$ through the above Gram--Schmidt procedure; the constant function is covered by $1=\sum_{r=1}^Q \idc[x_i=r]$.
The basis is not unique, but all $L^2(\+P)$ identities below depend only on the resulting orthogonal decomposition into constant and mean-zero parts, not on the particular choice.

Now, we put the space in each coordinate together to obtain a tensor product space. 
For each coordinate \(i\), let $H_i$ be the space $L^2(\+P^{(i)})$ over the alphabet $\supp(\+P^{(i)})$. Let $H_i^{\perp} = \{g\in H_i:\mathbb E_{\+P^{(i)}}[g]=0\}$ be the space of functions in $H_i$ that are orthogonal to the constant function.
Then
$
H_i=\mathrm{span}\{1\}\oplus H_i^{\perp},
$
where \(\psi_{i,1}\equiv 1\) spans the constant part, and
\(\psi_{i,2},\ldots,\psi_{i,Q}\) form an orthonormal basis of the
mean-zero part \(H_i^\perp\).
Since \(\+P=\+P^{(1)}\otimes\cdots\otimes \+P^{(n)}\), we may view
$
L^2(\+P)\cong H_1\otimes\cdots\otimes H_n .
$
Hence \(L^2(\+P)\) has the orthogonal decomposition
$
L^2(\+P)=\bigoplus_{S\subseteq[n]} H_S$ where $H_S=
(\bigotimes_{i\in S}H_i^\perp)
\otimes
(\bigotimes_{j\notin S}\mathrm{span}\{1\}).
$
The Efron--Stein component $f_S$ lies in $H_S$ because it depends only on
the coordinates in \(S\), and has mean zero in each of those coordinates.
Therefore, a natural orthonormal basis for \(H_S\) is given by
$
\psi_{S,a}(x)
=
\prod_{i\in S}\psi_{i,a_i}(x_i)$ for all
$a\in\{2,\ldots,Q\}^S$,
with the convention \(\psi_{\emptyset, \emptyset}\equiv 1\). Hence the expansion of $f_S$ in this basis is
$f_S = \sum_{a\in\{2,\ldots,Q\}^S} c_{S,a}\psi_{S,a}$, where $c_{S,a} = \langle f_S,\psi_{S,a}\rangle_{L^2(\+P)}$. Parseval's identity gives $\norm{f_S}_{2, \+P}^2 = \sum_{a \in \{2, \ldots, Q\}^S} c_{S, a}^2$.


For functions on $[Q]^n$, we use the same Fourier-weight notation as in the Boolean case, with the product decomposition replacing the Boolean Fourier decomposition:
\begin{align}\label{eq:product-fourier-weights}
\mathrm{W}^{=k}(f) = \sum_{\abs{S} = k} \norm{f_S}_{2, \+P}^2, \qquad
\norm{f_S}_{2, \+P}^2 = \oE_{x \sim \+P}[f_S(x)^2].
\end{align}
Similarly, $\mathrm{W}^{\ge k}(f)=\sum_{\abs{S}\ge k}\norm{f_S}_{2,\+P}^2$ and $\mathrm{W}^{>k}(f)=\sum_{\abs{S}>k}\norm{f_S}_{2,\+P}^2$.
Furthermore, the level-$k$ truncation is $f^{\le k} = \sum_{\abs{S} \le k} f_S$, and orthogonality gives
$\oE_{x \sim \+P}[(f(x) - f^{\le k}(x))^2] = \mathrm{W}^{>k}(f)$.
Since each $\psi_{i,q}$ is a linear combination of the single-coordinate indicators $\idc[x_i=r]$, each product basis function $\psi_{S,a}$ is an indicator polynomial whose monomials use only coordinates in $S$. Hence $f^{\le k}$ admits an indicator-polynomial representation with monomials of support size at most $k$, i.e. $\codeg(f^{\le k})\le k$.

\paragraph{Selector reduction.}
We use a standard random two-point embedding of the Boolean cube into the product space. For each coordinate, sample two independent values from $\+P^{(i)}$, with a Boolean selector choosing between them. After conditioning on these sampled pairs, the function $f$ becomes a Boolean function on the selector variables; averaging over the pairs then recovers information about the product Fourier levels of $f$.

Fix a function $f : [Q]^n \to \{\pm 1\}$ and a product distribution $\+P = \+P^{(1)} \otimes \cdots \otimes \+P^{(n)}$ on $[Q]^n$.
Suppose $Z^- = (Z^-_1, \ldots, Z^-_n)$ and $Z^+ = (Z^+_1, \ldots, Z^+_n)$ are two independent random variables distributed according to $\+P$. Let $Z = (Z^-, Z^+)$.
For any $y \in \{\pm 1\}^n$, define $Z^y = (Z^y_1, \ldots, Z^y_n)$ by $Z^y_i = Z^-_i$ if $y_i = -1$ and $Z^+_i$ if $y_i = +1$.
Let $f^Z : \{\pm 1\}^n \to \{\pm 1\}$ be the selector function defined by $f^Z(y) = f(Z^y)$.

\begin{lemma}[Circuit preservation]
  \label{lem:circuit-preservation}
  Suppose that $f: [Q]^n \to \{\pm 1\}$ is computed by an $\mathsf{AC}(d, s)$ circuit in the single-coordinate predicate model.
  For any $Z = (Z^-, Z^+)$, the function $f^Z : \{\pm 1\}^n \to \{\pm 1\}$ can be computed by an $\mathsf{AC}(d, s)$ circuit in the standard Boolean model.
\end{lemma}
\begin{proof}
  Each single-coordinate predicate $\idc[x_i \in A]$ reduces to one of $0$, $1$, $\idc[y_i = +1]$, or $\idc[y_i = -1]$. Substituting these for the input predicates preserves the circuit's depth and size.
\end{proof}

\begin{lemma}[Selector Fourier weights]
  \label{lem:selector-fourier-weights}
  For every integer $r \ge 0$, it holds that
  \[
  \oE_Z[\mathrm{W}^{\ge r}(f^Z)] = \sum_{S \subseteq [n]} \norm{f_S}_{2, \+P}^2 \cdot \oPr[\textnormal{Bin}(\abs{S}, 1/2) \ge r].
  \]
\end{lemma}
\begin{proof}
  Recall the basis representation from the product Fourier decomposition: $f = \sum_{S\subseteq [n]}f_S$, where $f_S = \sum_{a\in \{2,\ldots,Q\}^S} c_{S, a} \psi_{S, a}$ and $\psi_{S, a} = \prod_{i\in S} \psi_{i, a_i}$.
  Fix $S \subseteq [n]$ and $a \in \{2,\ldots,Q\}^S$, and write $h_{S,a}^Z(y) = \psi_{S,a}(Z^y)$.
  For $i \in S$, set $A_i^a = (\psi_{i,a_i}(Z_i^+) + \psi_{i,a_i}(Z_i^-))/2$ and $B_i^a = (\psi_{i,a_i}(Z_i^+) - \psi_{i,a_i}(Z_i^-))/2$.
  Then $h_{S,a}^Z(y) = \prod_{i \in S}(A_i^a + y_i B_i^a)$.
  One can directly expand the product to get $h_{S,a}^Z(y) = \sum_{U \subseteq S}\prod_{i \in U} B_i^a \prod_{i \in S \setminus U} A_i^a \prod_{i \in U} y_i$.
  Thus the Boolean Fourier coefficient of $h_{S,a}^Z$ at $U$ vanishes unless $U \subseteq S$; when $U \subseteq S$, it equals
  $\prod_{i \in U} B_i^a \prod_{i \in S \setminus U} A_i^a$.

  The following consequence of coordinate-wise orthogonality plays a key role.
  Since the nonconstant basis functions are mean zero and orthonormal, for $a_i,b_i \in \{2,\ldots,Q\}$ one has
  $\oE_Z[A_i^a A_i^b] = \oE_Z[B_i^a B_i^b] = \idc[a_i=b_i] / 2$, while
  $\oE_Z[A_i^a B_i^b] = \oE_Z[A_i^a] = \oE_Z[B_i^a] = 0$.
  By independence across coordinates, this implies
  \[
  \oE_Z\left[\widehat{h_{S,a}^Z}(U)\widehat{h_{S',a'}^Z}(U)\right] = 2^{-\abs{S}}\idc[S = S']\idc[a = a']\idc[U \subseteq S].
  \]
  We now explain the above identity. Recall that $\widehat{h_{S,a}^Z}(U) = \prod_{i \in U} B_i^a \prod_{i \in S \setminus U} A_i^a \idc[U \subseteq S]$ and $\widehat{h_{S',a'}^Z}(U)$ has a similar form. It is clear that the product is zero unless $U \subseteq S$ and $U \subseteq S'$. Next, if $S \neq S'$, then the product is zero because $\oE_Z[A_i^a] = \oE_Z[B_i^a] = 0$ for $i \in (S \setminus S') \cup (S' \setminus S)$. Finally, the factor $2^{-|S|}\idc[a = a']$ comes from $\oE_Z[A_i^a A_i^b] = \oE_Z[B_i^a B_i^b] = \idc[a_i=b_i] / 2$ for $i \in S$.

  Expanding $f^Z = \sum_{S \subseteq [n]}\sum_{a \in \{2,\ldots,Q\}^S} c_{S,a} h_{S,a}^Z$ and taking the Fourier transform of both sides gives
  \[
  \widehat{f^Z}(U) = \sum_{S \subseteq [n]}\sum_{a \in \{2,\ldots,Q\}^S} c_{S,a} \widehat{h_{S,a}^Z}(U).
  \]
  Consider the expectation of the product $\oE_Z[\widehat{f^Z}(U)^2]$ for $U\subseteq [n]$.
  The identity above removes all cross terms. We have $\oE_Z[\widehat{f^Z}(U)^2] = \sum_{S \supseteq U}\sum_{a \in \{2,\ldots,Q\}^S} 2^{-\abs{S}} c_{S,a}^2$.
  Summing over all $U$ with $\abs{U} \ge r$ and switching the order of summation gives
  $$\oE_Z[\mathrm{W}^{\ge r}(f^Z)] = \sum_{S \subseteq [n]}\Big(\sum_{a \in \{2,\ldots,Q\}^S} c_{S,a}^2\Big)\sum_{\substack{U \subseteq S: \abs{U} \ge r}} 2^{-\abs{S}}.$$
  The lemma follows from the definition of $\norm{f_S}_{2, \+P}^2$ and the fact that $\sum_{U \subseteq S, \abs{U} \ge r} 2^{-\abs{S}}$ is the probability that a $\textnormal{Bin}(\abs{S}, 1/2)$ random variable is at least $r$.
\end{proof}

\begin{proof}[Proof of \pref{thm:low-degree-beyond-boolean}]
If $T=n$, take $g=f$.
The indicator monomials span all functions on $[Q]^n$, so $\codeg(g)\le n=T$; the approximation error is zero, and $\norm{g}_{L^\infty(\+P)}\le 1\le (2\e n/T)^T$.

Assume henceforth that $T<n$, which implies $T=\ceil{16C_d\log^{d-1}(2s)\log(2/\eps)}$.
Set $r=\floor{T/2}$.
Then the elementary bound $\floor{\ceil{x}/2}\ge x/3$ for $x\ge 6$ gives $r\ge 16C_d\log^{d-1}(2s)\log(2/\eps) / 3$.

We first prove the tail bound $\mathrm{W}^{\ge 2r}(f) \le \eps$.
Fix $Z = (Z^-, Z^+)$ as in the selector reduction.
By \pref{lem:circuit-preservation}, $f^Z$ is computed by a standard Boolean $\mathsf{AC}(d,s)$ circuit.
Tal's bound (\Cref{thm:tal-fourier-tail}), averaged over $Z$, gives $\oE_Z[\mathrm{W}^{\ge r}(f^Z)] \le 2 \cdot 2^{-r/(C_d\log^{d-1}(2s))} \le \eps / 2$.
By selector identity (\Cref{lem:selector-fourier-weights}),
$$\oE_Z[\mathrm{W}^{\ge r}(f^Z)] = \sum_{S \subseteq [n]} \norm{f_S}_{2,\+P}^2 \cdot \oPr[\textnormal{Bin}(\abs{S},1/2) \ge r].$$
By symmetry of the binomial distribution, every term with $\abs{S} \ge 2r$ has $\oPr[\textnormal{Bin}(\abs{S},1/2) \ge r] \ge 1/2$.
It follows from~\eqref{eq:product-fourier-weights} that 
$$\mathrm{W}^{\ge 2r}(f) = \sum_{\abs{S} \ge 2r} \norm{f_S}_{2,\+P}^2 \le 2\,\oE_Z[\mathrm{W}^{\ge r}(f^Z)] \le \eps.$$
Now take $g=f^{\le T}$.
Since $2r\le T$, orthogonality gives the approximation error bound 
$$\oE_{x\sim\+P}[(f(x)-g(x))^2]=\mathrm{W}^{>T}(f)\le \mathrm{W}^{\ge 2r}(f)\le\eps.$$

To bound the coordinate degree, recall that $f_S=\sum_{a\in\{2,\ldots,Q\}^S}c_{S,a}\psi_{S,a}$.
Each product basis function expands in the indicator basis as
\[
\psi_{S,a}(x)
= \prod_{i\in S}\Big(\sum_{q_i\in[Q]}\psi_{i,a_i}(q_i)\idc[x_i=q_i]\Big)
= \sum_{q_S\in[Q]^S}\Big(\prod_{i\in S}\psi_{i,a_i}(q_i)\Big)
   \prod_{i\in S}\idc[x_i=q_i].
\]
Thus $\psi_{S,a}$ is an indicator polynomial supported on $S$. Since $g=\sum_{\abs{S}\le T}f_S$, we have $\codeg(g)\le T$.

It remains to bound $\norm{g}_{L^\infty(\+P)}$.
The averaging operator satisfies $\norm{\-E_i h}_{L^\infty(\+P)} \le \norm{h}_{L^\infty(\+P)}$, while $\norm{\Delta_i h}_{L^\infty(\+P)} \le 2\norm{h}_{L^\infty(\+P)}$.
Since $\abs{f} \le 1$, the definition of $f_S$ in~\eqref{eq:product-efron-stein-decomposition} implies $\norm{f_S}_{L^\infty(\+P)} \le 2^{\abs{S}}$.
Thus,
\[
\norm{g}_{L^\infty(\+P)}
\le \sum_{\abs{S} \le T} \norm{f_S}_{L^\infty(\+P)}
\le \sum_{k = 0}^{T} \binom{n}{k}2^k
\le \left(\frac{2\e n}{T}\right)^T. \qedhere
\]
\end{proof}

\section{Low-Degree Approximation under Gibbs Distributions}
\label{sec:framework}
In this section, we provide a framework for transferring low-degree approximations from product sampler
randomness to a Gibbs distribution using local samplers.
We first define the $\samp$-$\insamp$ pair,
which implements this transfer at the level of distributions.
Then we use local samplers to
approximate the $\samp$ and $\insamp$ procedures, 
which gives the desired low-degree approximation under $\mu$.
This approximation entails viewing local samplers as automata and truncating their paths.

\subsection{Sampler-Inverter Pair via Glauber Dynamics}

First, we introduce the update simulator, 
which realizes the $T$-step Glauber dynamics along a deterministic update sequence through an extra mark sequence and deterministic update rules.

\begin{definition}
  \label{def:update-simulator}
  Given a time horizon $T$ and a deterministic update sequence
  $\*u=(u_1,\ldots,u_T)\in V^T$ of Glauber dynamics,
  an \emph{update simulator} consists of a sequence
  $((\+P_t,\psi_t))_{t\in[T]}$, where each $\+P_t$ is a distribution over
  the same finite \emph{mark} alphabet $[Q]$, and each $\psi_t$ is a deterministic
  update rule mapping, from the tuple
  $(\rho,\tau)\in [Q]\times\set{\pm 1}^{N(u_t)}$ to a spin
  $\psi_t(\rho,\tau)\in\set{\pm 1}$, such that
  \begin{equation*}
    \forall t\in[T],\ \tau\in\set{\pm 1}^{N(u_t)},\ c\in\set{\pm 1},
    \quad
    \Pr[\rho\sim\+P_t]{\psi_t(\rho,\tau)=c}=\mu_{u_t}^\tau(c).
  \end{equation*}
\end{definition}

At the $t$-th step of Glauber dynamics, the spin at the updating vertex $u_t$ is resampled conditional on the neighboring configuration $\tau$.
The update simulator implements this randomized step by first drawing a mark $\rho_t\sim\+P_t$ and then applying the deterministic rule $\psi_t$.
Thus all randomness in the update is carried by the mark $\rho_t$, whose possible values form the mark alphabet $[Q]$.
We will specify $[Q]$ explicitly in the applications.
We also remark that if the initial state is in $\supp(\mu)$ and for any $t\in [T]$, $\rho_t \in \supp(\+P_t)$, then this dynamics always stays in $\supp(\mu)$.

The update simulator realizes the $T$-step Glauber dynamics along $\*u$
in the following sense:
For a mark sequence
$\*\rho=(\rho_1,\ldots,\rho_T)\in[Q]^T$ and an initial configuration $\sigma\in\supp(\mu)$, define the sequence of configurations
$\Psi_t^{\*u}(\*\rho,\sigma) \in \{\pm 1\}^V$ recursively by
$\Psi_0^{\*u}(\*\rho,\sigma)=\sigma$, and for
$t=1,\ldots,T$, $\Psi_t^{\*u}(\*\rho,\sigma)$ is obtained from
$\Psi_{t-1}^{\*u}(\*\rho,\sigma)$ by updating only the site $u_t$ while all other vertices remain unchanged:
\begin{equation*}
  \Psi_t^{\*u}(\*\rho,\sigma)(v)
  =\begin{cases}
  \psi_t(\rho_t,\Psi_{t-1}^{\*u}(\*\rho,\sigma)(N(u_t))) &\text{if } v=u_t; \\
  \Psi_{t-1}^{\*u}(\*\rho,\sigma)(v) &\text{if } v\ne u_t.\\
  \end{cases}
\end{equation*}
If $\*\rho\in \supp(\bigotimes_{t=1}^T \+P_t)$ then both $\Psi_t^{\*u}(\*\rho,\sigma)$ and $\Psi_{t-1}^{\*u}(\*\rho,\sigma)$ are configurations in $\supp(\mu)$ since the initial configuration $\sigma\in \supp(\mu)$.
For any set $S \subseteq V$, we write $\Psi_t^{\*u}(\*\rho,\sigma)(S) \in \{\pm 1\}^S$ for the restriction of $\Psi_t^{\*u}(\*\rho,\sigma)$ to $S$.

In the remainder of the paper,
we fix $\*u$ as the cyclic scan,
i.e., $u_t=v_{((t-1)\bmod n)+1}$ for all $t=1,\ldots,T$;
and write $\*u^\leftarrow=(u_T,\ldots,u_1)$ for the reversed scan of $\*u$.
The reversed scan is paired with the reversed simulator
$((\+P_{T-s+1},\psi_{T-s+1}))_{s\in[T]}$.
Write
\begin{equation*}
  \+P_{1:T}\defeq\bigotimes_{t=1}^T\+P_t,
  \qquad
  \+P_{T:1}\defeq\bigotimes_{s=1}^T\+P_{T-s+1},
\end{equation*}
where the $s$-th coordinate under $\+P_{T:1}$ is distributed according to
$\+P_{T-s+1}$ and is aligned with the update site $u_{T-s+1}$ in the reversed
scan.
Write $\Psi_t^{\rightarrow}$ for the trajectory generated by
$\*u$ and $((\+P_t,\psi_t))_{t\in[T]}$, and write
$\Psi_s^{\leftarrow}$ for the trajectory generated by $\*u^\leftarrow$ and
$((\+P_{T-s+1},\psi_{T-s+1}))_{s\in[T]}$. Thus the $s$-th reverse step applies
$\psi_{T-s+1}$ with marking distribution $\+P_{T-s+1}$ at the site $u_{T-s+1}$.

Fix a reference configuration $\sigma^\star\in\supp(\mu)$ once and
for all.
The map $\samp$ takes a mark sequence $\*\rho\in[Q]^T$ as input and outputs a
configuration in $\set{\pm 1}^V$ by simulating the forward trajectory
$\Psi_t^{\rightarrow}(\*\rho,\sigma^\star)$ for $t=1,\ldots,T$. 
\Cref{alg:abstract-sampler} returns a random configuration generated by the $T$-step Glauber dynamics if the input $\*\rho$ is drawn from the product distribution $\+P_{1:T}$.

\begin{algorithm}[htbp]
  \caption{$\samp(\*\rho)$}
  \label{alg:abstract-sampler}
  \KwIn{a mark sequence $\*\rho\in[Q]^T$}
  \KwOut{a configuration in $\set{\pm 1}^V$}

  \For{$t$ \textnormal{from} $1$ \textnormal{to} $T$}{
    Set $\Psi_t^\rightarrow(\*\rho,\sigma^\star)\gets \Psi_{t-1}^\rightarrow(\*\rho,\sigma^\star)$\;
    Update $\Psi_t^\rightarrow(\*\rho,\sigma^\star)(u_t)\gets
    \psi_t(\rho_t,\Psi_{t-1}^\rightarrow(\*\rho,\sigma^\star)(N(u_t)))$\tcp*{Glauber update at $u_t$}
  }

  \Return{$\Psi_T^\rightarrow(\*\rho,\sigma^\star)$}\;
\end{algorithm}

We next define the inverse sampling procedure $\insamp$, shown in
\Cref{alg:abstract-invsample}.
Given a configuration $y\in\supp(\mu)$ and an auxiliary random bit string
$\*r$, it outputs a mark sequence $\*\rho\in[Q]^T$.
The goal is to approximate the law of $\*\rho\sim\+P_{1:T}$ conditional on
$\Psi_T^\rightarrow(\*\rho,\sigma^\star)=y$.

The randomness is split as
$\*r=(\*r^{\mathsf{seq}},\*r^{\mathsf{con}})$, where
$\*r^{\mathsf{con}}=(\*r_t^{\mathsf{con}})_{t\in[T]}$.
The string $\*r^{\mathsf{seq}}$ is used to sample the reverse mark sequence
in \pref{line:inv-sample-eta}, while $\*r_t^{\mathsf{con}}$ is used for the
conditional resampling step at time $t$ in \pref{line:mark-sampling}.
These are the only independent sampling calls in the procedure; once $\*r$ is
fixed, $\insamp(y,\*r)$ is deterministic.
For now, one may regard each of $\*r^{\mathsf{seq}}$ and
$\*r_t^{\mathsf{con}}$ as an infinite sequence of independent unbiased bits.
In the later proofs, we will show that these random sources can be implemented
by finite-support random variables.

Operationally, $\insamp$ first samples a trajectory in the reversed scan from
$y$ to some configuration $\sigma_0$ by applying the reversed simulator
sequence $((\+P_{T-s+1},\psi_{T-s+1}))_{s\in[T]}$; it then samples a mark
sequence $\*\rho$ that is consistent with the sampled trajectory.

\begin{algorithm}[htbp]
  \caption{$\insamp(y,\*r)$}
  \label{alg:abstract-invsample}
  \KwIn{a configuration $y\in\supp(\mu)$ and a random bit string
  $\*r = (\*r^{\mathsf{seq}},(\*r_t^{\mathsf{con}})_{t\in[T]})$}
  \KwOut{a mark sequence $\*\rho\in[Q]^T$}

  Sample $\*\eta\sim\+P_{T:1}$ \tcp*{using $\*r^{\mathsf{seq}}$}\label{line:inv-sample-eta}
  \For{$s$ \textnormal{from} $1$ \textnormal{to} $T$}{
    Set $\sigma_{T-s}\gets \Psi_{s}^{\leftarrow}(\*\eta,y)$\;\label{line:reverse-trajectory}
  }
  \For{$t$ \textnormal{from} $1$ \textnormal{to} $T$}{
    Sample $\rho_t\sim
    \+P_t$ conditional on $\psi_t(\rho_t,\sigma_{t-1}(N(u_t)))=\sigma_t(u_t)$
    \tcp*{using $\*r_t^{\mathsf{con}}$}\label{line:mark-sampling}
  }
  \Return{$\*\rho$}
\end{algorithm}

The following lemma holds due to the reversibility of the Glauber dynamics. The lemma says that given a random $y \sim \mu$, the $\insamp$ procedure generates a mark sequence $\*\rho \sim \+P_{1:T}$.

\begin{lemma}
  \label{lem:general-rho-invert}
  Given a time horizon $T$,
  for every $\*\rho\in[Q]^T$ and every $y\in\supp(\mu)$,
  let $\pre_y(\*\rho)\defeq\set{\sigma\in \supp(\mu) \given \Psi_T^\rightarrow(\*\rho,\sigma)=y}$
  be the preimage of $y$ under the forward dynamics driven by $\*\rho$.
  Then
  \begin{equation}
    \label{eq:invsamp-probability}
    \Pr[\*r]{\insamp(y,\*r)=\*\rho} = \frac{\mu\tp{\pre_y(\*\rho)}}{\mu(y)} \+P_{1:T}(\*\rho).
  \end{equation}
  Moreover,
  \begin{equation}
    \label{eq:invsamp-probability-y}
    \Pr[y\sim\mu,\*r]{\insamp(y,\*r)=\*\rho} = \+P_{1:T}(\*\rho).
  \end{equation}
\end{lemma}

\begin{proof}
  In $\insamp(y, \*r)$, where $\*r$ is a random bit string,
  for a specific Glauber trajectory $(\sigma_T,\ldots,\sigma_0)$ with $\sigma_T=y$,
  the probability that it was generated by \pref{line:reverse-trajectory} of \Cref{alg:abstract-invsample} is
  \begin{equation*}
    \prod_{t=1}^{T} P_{u_t}(\sigma_t,\sigma_{t-1})
    \overset{\text{by }\pref{eq:stationary-path-reversal}}{=}
    \frac{\mu(\sigma_0)}{\mu(y)}
    \prod_{t=1}^T P_{u_t}(\sigma_{t-1},\sigma_t)
    =\frac{\mu(\sigma_0)}{\mu(y)}
    \prod_{t=1}^T \mu_{u_t}^{\sigma_{t-1}(N(u_t))}(\sigma_t(u_t)).
  \end{equation*}
  Given this trajectory, the probability that \pref{line:mark-sampling}
  generates a specific mark sequence $\*\rho$ is
  \begin{equation*}
    \prod_{t=1}^T
    \frac{
      \+P_t(\rho_t)
      \idc\interval{\psi_t\tuple{\rho_t,\sigma_{t-1}(N(u_t))}=\sigma_t(u_t)}
    }{
      \mu_{u_t}^{\sigma_{t-1}(N(u_t))}(\sigma_t(u_t))
    }.
  \end{equation*}
  The indicators in the above equation are all equal to one
  if and only if $(\sigma_0,\ldots,\sigma_T)$ is the trajectory
  $(\Psi_0^{\rightarrow}(\*\rho,\sigma_0),\ldots,
  \Psi_T^{\rightarrow}(\*\rho,\sigma_0))$,
  which depends only on $\sigma_0$ (for fixed $\*\rho$).
  Then, in \Cref{alg:abstract-invsample}, $\Pr[\*r]{\insamp(y, \*r) = \*\rho}$ can be written as
  \begin{align}\label{eq:invsamp-probability-proof}
    &\sum_{(\sigma_T,\ldots,\sigma_0):\sigma_T=y}
    \Pr[\*r]{\text{\pref{line:reverse-trajectory} generates }(\sigma_T,\ldots,\sigma_0)}
    \Pr[\*r]{\text{\pref{line:mark-sampling} generates }\*\rho
    \given (\sigma_T,\ldots,\sigma_0)}
  \end{align}
  By the analysis above, we only need to enumerate over all possible starting configurations $\sigma_0$ that are compatible with the mark sequence $\*\rho$ and the final configuration $y$. Then all configurations $\sigma_t = \Psi_t^{\rightarrow}(\*\rho,\sigma_0)$ for $t=1,\ldots,T$ are uniquely fixed by $\sigma_0$ and the mark sequence $\*\rho$.
  Therefore,
  \begin{align*}
   \eqref{eq:invsamp-probability-proof} =& \; \sum_{\substack{\sigma_0\in \pre_y(\*\rho) \\ \forall t \in [T]:\sigma_t = \Psi_t^{\rightarrow}(\*\rho,\sigma_0)}} \frac{\mu(\sigma_0)}{\mu(y)}
    \tuple{\prod_{t=1}^T \mu_{u_t}^{\sigma_{t-1}(N(u_t))}(\sigma_t(u_t))}
    \tuple{\prod_{t=1}^T \frac{\+P_t(\rho_t)}{\mu_{u_t}^{\sigma_{t-1}(N(u_t))}(\sigma_t(u_t))}}\\
    =& \; \sum_{\sigma_0\in \pre_y(\*\rho)} \frac{\mu(\sigma_0)}{\mu(y)} \prod_{t = 1}^T \+P_t(\rho_t)
    =\frac{\+P_{1:T}(\*\rho)}{\mu(y)} \mu\tp{\pre_y(\*\rho)}.\nonumber
  \end{align*}

  To prove \pref{eq:invsamp-probability-y}, we observe that
  the sets $\pre_y(\*\rho)$ partition $\supp(\mu)$ as $y$ ranges over $\supp(\mu)$,
  since $\Psi_T^\rightarrow(\*\rho,\cdot)$ is a deterministic mapping for fixed $\*\rho$.
  Therefore,
  \begin{align*}
      \Pr[y\sim\mu,\*r]{\insamp(y,\*r)=\*\rho}
      &= \sum_{y\in\supp(\mu)}\mu(y)
      \Pr[\*r]{\insamp(y,\*r)=\*\rho} \\
      &= \+P_{1:T}(\*\rho)
      \sum_{y\in\supp(\mu)}\mu\tp{\pre_y(\*\rho)}
      = \+P_{1:T}(\*\rho).\qedhere
  \end{align*}
\end{proof}

\subsection{Low-degree approximation via good local samplers}

\Cref{alg:abstract-sampler} and \Cref{alg:abstract-invsample} define the
sampling and inverse sampling procedures for a graphical model $(G,\mu)$.
For our purposes, we need low-degree approximations to these procedures. This
requires local implementations: to compute the output at a site $v$, the
sampler should inspect only a small number of coordinates of the input mark
sequence $\*\rho$, and the inverse sampler should satisfy an analogous
locality property.

To formalize this locality, we represent local samplers by automata,
following~\cite{CCLZ26Subquadratic}. This representation isolates the finite
local information read by a sampler and allows us to truncate automaton paths
when constructing low-degree approximations to $\samp$ and $\insamp$.

\begin{definition}[Automaton]
  \label{def:automaton}
  An automaton is a tuple $(\+I,\+S,\texttt{init},\bot,\delta)$, where
  \begin{itemize}
    \item $\+I$ is the input space;
    \item $\+S$ is the state space, and $\texttt{init}\in\+S$ is the initial
    state;
    \item $\bot=\set{\bot_-,\bot_+}\subseteq\+S$ is the set of absorbing
    states. Reaching $\bot_-$ produces output $-1$, while reaching $\bot_+$
    produces output $+1$;
    \item $\delta:\+S\setminus\bot\times\+I\to\+S$ is the deterministic
    transition map. On input $I\in\+I$, the automaton moves from a
    non-absorbing state $x$ to $\delta(x,I)$.
  \end{itemize}
\end{definition}

For a fixed input $I\in\+I$, the run of the automaton is the path
$(x_\ell)_{\ell\ge0}$ defined by $x_0=\texttt{init}$ and
$x_{\ell+1}=\delta(x_\ell,I)$ whenever $x_\ell\notin\bot$. The stopping time is
$L_{\mathrm{stop}}(I)\defeq\inf\set{\ell\ge0\cmid x_\ell\in\bot}$. If
$L_{\mathrm{stop}}(I)<\infty$, then the automaton terminates and outputs the
sign associated with the absorbing state it reaches. A path is
\emph{realizable} if it arises from some input, and it is \emph{accepting} if
it is realizable and ends in an absorbing state. When the input is clear, we
suppress it from the notation.

In our applications, the input is a pair $(\*\rho,\sigma)$, where
$\*\rho\in[Q]^T$ is a mark sequence and $\sigma\in\supp(\mu)$ is an initial
configuration. The automaton may query coordinates of both objects during its
run; these objects are treated as part of the input and are therefore not
included in the state notation.

\begin{definition}[Local sampler for a graphical model]
  \label{def:local-sampler}
  Let $(G,\mu)$ be a graphical model with vertex set $V$. 
  Given a time horizon $T$, a deterministic cyclic update sequence $\*u\in V^T$, and an
  update simulator $((\+P_t,\psi_t))_{t\in[T]}$ along $\*u$, a family of
  automata $\set[t\in {[T]},\,v\in V]{\+M_{t,v}^{\*u}}$ is a
  \emph{local sampler} for $(G,\mu)$ along $\*u$, with respect to
  $((\+P_t,\psi_t))_{t\in[T]}$, if for every $1\le t\le T$ and every
  $v\in V$, on input $(\*\rho,\sigma)\in[Q]^T\times\supp(\mu)$, the automaton
  $\+M_{t,v}^{\*u}$ outputs $\Psi_t^{\*u}(\*\rho,\sigma)(v)$ whenever it
  terminates.
\end{definition}

The purpose of these automata is to provide coordinate-wise access to the
Glauber trajectory. To recover the spin at a single vertex $v$ at time $t$,
the automaton $\+M_{t,v}^{\*u}$ follows only the dependencies that are needed
to determine this spin, querying the relevant coordinates of the mark sequence
and, if necessary, of the initial configuration. Thus, unlike
\Cref{alg:abstract-sampler}, it does not have to simulate the whole chain or
construct the full configuration $\Psi_t^{\*u}(\*\rho,\sigma)$. This locality
is what later allows us to control the degree of the polynomial
approximations. 

Fix an automaton.
We also record the coordinates queried by each state. For every
$x\in\+S\setminus\bot$, let $\poscon(x)$ be the set of coordinates of
$\sigma$ queried at $x$, and let $\posseq(x)$ be the set of coordinates of
$\*\rho$ queried at $x$. Equivalently, if two inputs
$(\*\rho,\sigma)$ and $(\*\rho',\sigma')$ agree on all coordinates in
$\posseq(x)$ and $\poscon(x)$, then they induce the same transition from
$x$. These sets will be used to state the local dependency requirements in
\pref{cond:good-local-sampler}.

\begin{condition}[Good local sampler with respect to simulator]
  \label{cond:good-local-sampler}
  Let $B\ge1$, $C>0$, $\alpha\in(0,1)$.
  Given a simulator $((\+P_t,\psi_t))_{t\in[T]}$, 
  the graphical model $(G,\mu)$ satisfies the $\goodsampler$-\emph{good local sampler condition} along $\*u$ with respect to $((\+P_t,\psi_t))_{t\in[T]}$
  if there exist local samplers $\set[t\in {[T]},\,v\in V]{\+M_{t,v}^{\*u}}$ for $(G,\mu)$ along $\*u$, with respect to $((\+P_t,\psi_t))_{t\in[T]}$, 
  such that for any $t \in [T]$ and $v \in V$, $\+M_{t,v}^{\*u}$, the following holds.
  \begin{enumerate}[leftmargin=*]

    \item \emph{Local dependency.}
    \begin{enumerate}
      \item \label{item:local-dependency-1}
      For every state $x\in\+S\setminus\bot$,
      $|\poscon(x)|\le B$ and $|\posseq(x)|\le B$.

      \item \label{item:local-dependency-3}
      For every realizable path $(x_0,\ldots,x_\ell)$ with $\ell \le \lfloor (t-n) / n \rfloor$, 
      $\poscon(x_j)=\emptyset$ for all $j\le \ell$.
    \end{enumerate}

    \item \label{item:exponential-decay} \emph{Exponential decay.}
    For every initial configuration $\sigma\in\supp(\mu)$ and every
    $\ell\ge0$,
    \begin{equation*}
      \Pr[\*\rho\sim\+P_{1:T}]
      {L_{\mathrm{stop}}(\*\rho,\sigma)>\ell}
      \le C\alpha^\ell .
    \end{equation*}
  \end{enumerate}
\end{condition}

We now explain these conditions. 
For local dependency, the first requirement says that
each state queries only a bounded number of coordinates of the input mark
sequence and the initial configuration. The second requirement says that, in
$\+M_{t,v}^{\*u}$, every realizable run of length at most $\lfloor \frac{t - n}{n}\rfloor$ avoids querying
the initial configuration $\sigma$, although it may query the mark sequence
$\*\rho$. Equivalently, the first $\lfloor \frac{t - n}{n} \rfloor$ steps of the automaton are independent
of $\sigma$. Finally, the exponential-decay condition ensures that the
automata terminate quickly.

\color{black}

Building upon the good local sampler condition with respect to a certain simulator, 
we introduce the following bi-directional good local sampler condition.
Specifically, this requires the existence of a simulator such that the good local sampler condition holds simultaneously along $\*u$ with respect to the simulator, 
and along $\*u^\leftarrow$ with respect to its inverse.

\begin{condition}[Bi-directional good local sampler]
  \label{cond:bi-good-local-sampler}
  Let $B\ge1$, $C>0$, $\alpha\in(0,1)$.
  The graphical model $(G,\mu)$ satisfies the \emph{bi-directional $\goodsampler$-good local sampler condition} along $\*u$ if,
  for any time horizon $T\ge n$, 
  there exist a simulator $((\+P_t,\psi_t))_{t\in[T]}$ along $\*u$ such that $(G,\mu)$ satisfies the $\goodsampler$-\emph{good local sampler condition} along $\*u$ with respect to $((\+P_t,\psi_t))_{t\in[T]}$ and 
  along $\*u^\leftarrow$ with respect to $((\+P_{T-s+1},\psi_{T-s+1}))_{s\in[T]}$ simultaneously.
\end{condition}

\color{black}
The main theorem of this section shows that if a graphical model with Gibbs distribution $\mu$ satisfies certain bi-directional good local sampler condition, 
then any Boolean function in $\mathsf{AC}(d,s)$ has a low-degree polynomial approximation in $L^2(\mu)$.

\begin{theorem}[Low-degree approximation from a good local sampler]
  \label{thm:good-sampler-polynomial-approximation}
  Let $G=(V,E)$ be a graph with $|V|=n$ and maximum degree $\Delta$, and let $\mu$ be a Gibbs distribution on $\{\pm 1\}^V$.
  Suppose that, with mark domain $[Q]$, 
  $(G,\mu)$ satisfies the bi-directional $\goodsampler$-good local sampler condition along $\*u$, where $B, C, \alpha, Q$ are fixed constants independent of $n$.
  Then for every Boolean $f:\{\pm 1\}^V\to\{\pm 1\}$ computed by an $\mathsf{AC}(d,s)$ circuit with $s=n^{c}$ for a fixed constant $c$ and every $\eps\in(0,1)$, there is a polynomial $p:\{\pm 1\}^V\to\mathbb R$ such that
  \begin{align}\label{eq:good-sampler-polynomial-approximation}
    \E[y\sim\mu]{(f(y)-p(y))^2}\le \eps,
    \qquad
    \deg(p)\le \log^{O(d)}(n/\eps).
  \end{align}
  The hidden constant in $O(d)$ depends only on $B,C,\alpha,Q,\Delta,c$.
\end{theorem}

The rest of this section is devoted to the proof of this theorem. We begin
with the main idea in an idealized setting.
Suppose first that, instead of using the fixed initial configuration
$\sigma^\star$, the sampler in \Cref{alg:abstract-sampler} were initialized
from $\mu$, and then, for $\*\rho\sim\+P_{1:T}$, the output
$\samp(\*\rho)$ must be an exact sample from $\mu$.
Then approximating $f$ in $L^2(\mu)$ can be reduced to approximating
$f\circ\samp$ under the product distribution $\+P_{1:T}$.
Assume further that $f\circ\samp$ can be computed by a low-depth polynomial-size AC circuit in
the sense of~\Cref{def:ac-over-q-n}. By
\Cref{thm:low-degree-beyond-boolean}, there is a low-degree indicator
polynomial $g: \supp(\+P_{1:T}) \to \mathbb R$ such that
$\oE_{\*\rho\sim\+P_{1:T}}[\abs{f \circ \samp(\*\rho) - g(\*\rho)}^2] \leq \eps$.

At this point we have a low-degree polynomial on mark sequences, whereas
\Cref{eq:good-sampler-polynomial-approximation} requires a polynomial on
configurations. To transfer the approximation back from
$\*\rho\sim\+P_{1:T}$ to $y\sim\mu$, we use
\Cref{lem:general-rho-invert}: if $y\sim\mu$ and $\*r$ is an independent
random bit string, then $\insamp(y,\*r)$ has distribution $\+P_{1:T}$.
In the idealized exact-inversion setting, this gives
\begin{align*}
  \eps
  \geq \E[\*\rho\sim\+P_{1:T}]{\abs{f \circ \samp(\*\rho) - g(\*\rho)}^2} &= \E[y\sim\mu,\*r]{\vert f(\samp(\insamp(y,\*r))) - g(\insamp(y,\*r)) \vert^2} \\
  (\text{using } \samp(\insamp(y,\*r)) \approx y) \quad
  &\approx \E[y\sim\mu,\*r]{\vert f(y) - g(\insamp(y,\*r)) \vert^2},
\end{align*}
where we use an idealized approximation $\samp(\insamp(y,\*r)) \approx y$. In the formal proof, we will show that $\samp(\insamp(y,\*r)) = y$ holds with high probability over $y\sim\mu$ and $\*r$, and thus we can control the error in this step.
Averaging over $\*r$ shows that there exists a fixed string
$\*r^\star$ such that
\[
  \E[y\sim\mu]{\vert f(y) - g(\insamp(y,\*r^\star)) \vert^2} \lesssim \eps .
\]
Thus one would take the polynomial in
\eqref{eq:good-sampler-polynomial-approximation} to be
$p(y)=g(\insamp(y,\*r^\star))$. The degree of $p$ is low provided that
$\insamp$ itself has low degree.

To turn this idealized argument into a rigorous proof, we address the following challenges:
\begin{itemize}
  \item The initial configuration $\sigma^\star$ of
  \Cref{alg:abstract-sampler} is fixed rather than drawn from $\mu$. 
  Moreover, we need to formalize the approximation step $\samp(\insamp(y,\*r)) \approx y$.
  We handle
  this in \pref{sec:determining-mark-sequences} by introducing determining
  mark sequences.
  \item We need to show that $f\circ\samp$ can be computed by a low-depth,
  polynomial-size AC circuit, and that $\insamp$ has low degree. For this we
  introduce a truncated version $\appsamp$/$\appinsamp$ of $\samp$/$\insamp$ in
  \pref{sec:truncated-insamp}; we then show that the truncation error is small
  and that the truncated map satisfies the required complexity bounds.
  \item Finally, in \pref{sec:put-together} we combine these ingredients to
  prove \Cref{thm:good-sampler-polynomial-approximation}.
\end{itemize}


\subsection{Determining Mark Sequences}\label{sec:determining-mark-sequences}

The above idealization assumes that $\samp$ starts from $\sigma^\star \sim \mu$ and then it produces a perfect sample $\samp(\*\rho) \sim \mu$, for $\*\rho \sim \+P_{1:T}$. 
However, to make $\samp$ deterministic, 
we must fix the initial configuration $\sigma^\star$, rather than drawing it from $\mu$, and then the final configuration produced by $\samp$ will not exactly be $\mu$.
To overcome this issue,
we introduce the notion of \emph{determining} mark sequences.
Intuitively, a mark sequence is determining if the final configuration produced by $\Psi_T^\rightarrow(\*\rho,\cdot)$
does not depend on the initial configuration.
Then, any fixed initial configuration $\sigma^\star$
produces the same final configuration as a random initial configuration $\sigma\sim \mu$.
To achieve this transfer, 
we will require that with high probability,
a random mark sequence is determining.

\begin{definition}
  \label{def:determining-sequence}
  A mark sequence $\*\rho\in[Q]^T$ is \emph{determining} with respect
  to simulator $((\+P_t,\psi_t))_{t\in[T]}$
  if
  \begin{equation*}
    \forall \sigma,\sigma'\in\supp(\mu),\quad
    \Psi_T^\rightarrow(\*\rho,\sigma)=\Psi_T^\rightarrow(\*\rho,\sigma').
  \end{equation*}
\end{definition}

\begin{condition}
  \label{cond:almost-determining}
  A simulator $((\+P_t,\psi_t))_{t\in[T]}$ satisfies 
  the $\epsDet$-\emph{almost-determining} condition if
  \begin{equation}
    \label{eq:samp-almost-determining}
    \Pr[\*\rho \sim \+P_{1:T}]{\*\rho\text{ is not determining}} \le \epsDet.
  \end{equation}
\end{condition}

Later, in \pref{lem:good-local-sampler-determining}, we will establish \pref{cond:almost-determining} by showing that the good local sampler condition implies \pref{cond:almost-determining}.
We now compare approximation error before and after the sampler-inverter transformation, assuming \pref{cond:almost-determining}.
By \Cref{def:update-simulator}, the mark alphabet $[Q]$ is finite.
Although $\insamp$ outputs a mark sequence in $[Q]^T$, we will often use its
Boolean representation.
Since the marks already take values in $[Q]$, we apply the one-hot
encoding from \pref{def:one-hot-encoding} with $n=T$ and alphabet size $Q$.
Thus, for a mark sequence $\*\rho\in[Q]^T$, we write
\begin{equation*}
  \rho^{\mathsf{oh}}_{i,j}
  \defeq
  \idc\interval{\rho_i=j},
  \qquad
  \*\rho^{\mathsf{oh}}
  \defeq
  \tp{\rho^{\mathsf{oh}}_{i,j}}_{i\in[T],\,j\in[Q]}.
\end{equation*}
Thus $\*\rho^{\mathsf{oh}}\in\set{0,1}^{[T]\times[Q]}$ consists of $T$
one-hot blocks.

For the inverse sampler, we write
\begin{equation*}
  \begin{aligned}
    \insamp_{i,j}(y,\*r)
    \defeq
    \idc\interval{(\insamp(y,\*r))_i=j},\quad
    \insamp^{\mathsf{oh}}(y,\*r)
    \defeq
    \tp{\insamp_{i,j}(y,\*r)}_{i\in[T],\,j\in[Q]}.
  \end{aligned}
\end{equation*}
When composing a function on mark sequences with $\insamp^{\mathsf{oh}}$, we
use the composition convention for indicator functions introduced in
\pref{sec:low-degree-product}.

\begin{lemma}
  \label{lem:sample-inv-comparison}
  Assume \pref{cond:almost-determining}.
  Let $f:\set{\pm 1}^V\to \set{\pm 1}$ be a Boolean function.
  Suppose there exists a function $g:[Q]^T\to\^R$ such that
  \begin{equation*}
    \E[\*\rho\sim\+P_{1:T}]{\abs{f\circ\samp(\*\rho)-g(\*\rho)}^p}
    \le \varepsilon .
  \end{equation*}
  Then,
  \begin{equation*}
    \E[y\sim\mu,\*r]{\abs{f(y)-g\circ\insamp^{\mathsf{oh}}(y,\*r)}^p} \le \varepsilon + \tp{\norm{f}_{\infty}+\norm{g}_{L^\infty(\+P_{1:T})}}^p \cdot \epsDet.
  \end{equation*}
\end{lemma}

\begin{proof}
  Split the expectation according to whether the mark sequence output by $\insamp$
  is determining. On the non-determining event, suppose that
  $\insamp(y,\*r)=\*\rho$. Since $\*\rho\in\supp(\+P_{1:T})$, the corresponding contribution to the loss is
  $\abs{f(y)-g(\*\rho)}^p$, which is at most
  $\tp{\norm{f}_{\infty}+\norm{g}_{L^\infty(\+P_{1:T})}}^p$ by the triangle inequality and
  the definitions of the two sup norms. It remains only to bound the probability
  of this event. Summing \pref{eq:invsamp-probability-y} over all
  non-determining mark sequences shows that the output of $\insamp$, when
  $y\sim\mu$, is non-determining with probability exactly
  $\Pr[\*\rho\sim\+P_{1:T}]{\*\rho\text{ is not determining}}$. By
  \pref{eq:samp-almost-determining}, this probability is at most $\epsDet$.
  Hence the non-determining contribution is at most
  $\tp{\norm{f}_{\infty}+\norm{g}_{L^\infty(\+P_{1:T})}}^p\epsDet$.

  It remains to bound the determining part. If $\*\rho$ is determining, then
  $\Psi_T^\rightarrow(\*\rho,\cdot)$ is constant, and since
  $\samp(\*\rho)=\Psi_T^\rightarrow(\*\rho,\sigma^\star)$, we have
  $\pre_y(\*\rho)=\supp(\mu)$ when $y=\samp(\*\rho)$ and
  $\pre_y(\*\rho)=\emptyset$ when $y\neq\samp(\*\rho)$. Hence, by
  \pref{eq:invsamp-probability},
  \begin{equation*}
    \begin{aligned}
      &\sum_{y\in\supp(\mu)}\mu(y)
      \sum_{\substack{
        \*\rho\in[Q]^T\\
        \*\rho\text{ determining}
      }}
      \Pr[\*r]{\insamp(y,\*r)=\*\rho}\cdot\abs{f(y)-g(\*\rho)}^p \\
      =& \;
      \sum_{y\in\supp(\mu)}
      \sum_{\substack{
        \*\rho\in[Q]^T\\
        \*\rho\text{ determining}
      }}
      \+P_{1:T}(\*\rho)\idc[\samp(\*\rho)=y]
      \abs{f(y)-g(\*\rho)}^p \\
      =& \;
      \sum_{\substack{
        \*\rho\in[Q]^T\\
        \*\rho\text{ determining}
      }}
      \+P_{1:T}(\*\rho)
      \abs{f\circ\samp(\*\rho)-g(\*\rho)}^p \le
      \E[\*\rho\sim\+P_{1:T}]{\abs{f\circ\samp(\*\rho)-g(\*\rho)}^p}
      \le \varepsilon .
    \end{aligned}
  \end{equation*}
  Combining the two bounds proves the lemma.
\end{proof}

Recall that the local sampler $\set[t\in {[T]},\,v\in V]{\+M_{t,v}^{\*u}}$
for $(G,\mu)$ along $\*u$ is defined in \Cref{def:local-sampler}.
In the sequel, we only use the two update sequences fixed above: $\*u$ denotes
the cyclic scan, and $\*u^\leftarrow$ denotes its reversal. We write
$\+M_{t,v}^{\rightarrow}$ for $\+M_{t,v}^{\*u}$ and
$\+M_{t,v}^{\leftarrow}$ for $\+M_{t,v}^{\*u^\leftarrow}$.
The following lemma shows that the good local sampler condition implies
\pref{cond:almost-determining}.

\begin{lemma}
  \label{lem:good-local-sampler-determining}
  Fix $T\ge n$.
  Suppose that $(G,\mu)$ satisfies the
  $\goodsampler$-good local sampler condition along $\*u$ with respect to some simulator $((\+P_t,\psi_t))_{t\in[T]}$.
  Then \pref{cond:almost-determining} holds with
  $\epsDet\le nC\alpha^{\lfloor \frac{T-n}{n} \rfloor}$.
\end{lemma}

\begin{proof}
  Let $\sigma^\star$ be the fixed reference configuration. We first observe
  that if, for every $v\in V$, the automaton
  $\+M_{T,v}^{\rightarrow}$ terminates on input $(\*\rho,\sigma^\star)$
  within $\lfloor \frac{T - n}{n} \rfloor$ steps, then $\*\rho$ is determining.

  Indeed, fix $v\in V$ and suppose that the run of
  $\+M_{T,v}^{\rightarrow}$ on $(\*\rho,\sigma^\star)$ reaches an absorbing
  state by time $\lfloor \frac{T-n}{n} \rfloor$. By \pref{item:local-dependency-3} in
  \pref{cond:good-local-sampler}, every state encountered before termination
  has $\poscon=\emptyset$. Hence this run does not query the initial
  configuration; for the same mark sequence $\*\rho$, it follows the same path
  and produces the same output for every $\sigma\in\supp(\mu)$. Since
  $\+M_{T,v}^{\rightarrow}$ is a local sampler, this output is
  $\Psi_T^\rightarrow(\*\rho,\sigma)(v)$ whenever the automaton terminates.
  Thus $\Psi_T^\rightarrow(\*\rho,\sigma)(v)$ is independent of $\sigma$.
  As this holds for every $v\in V$, the entire configuration
  $\Psi_T^\rightarrow(\*\rho,\sigma)$ is independent of $\sigma$, and so
  $\*\rho$ is determining in the sense of \pref{def:determining-sequence}.

  Consequently, if $\*\rho$ is not determining, then for some $v\in V$ the automaton $\+M_{T,v}^{\rightarrow}$ has stopping time greater than $ \lfloor \frac{T - n}{n} \rfloor$ on input $(\*\rho,\sigma^\star)$. By the exponential decay condition and a union bound over all vertices $v\in V$,
  \begin{equation*}
    \Pr[\*\rho\sim\+P_{1:T}]
    {\*\rho\text{ is not determining}}
    \le
    \sum_{v\in V}
    \Pr[\*\rho\sim\+P_{1:T}]
    {L_{\mathrm{stop}}(\*\rho,\sigma^\star)>{\left\lfloor \frac{T - n}{n} \right\rfloor}}
    \le nC\alpha^{\lfloor \frac{T - n}{n} \rfloor}.
  \end{equation*}
  This is precisely \pref{cond:almost-determining} with
  $\epsDet\le nC\alpha^{\lfloor \frac{T - n}{n} \rfloor}$.
\end{proof}

The above proof uses the second local-dependency requirement and exponential decay in \pref{cond:good-local-sampler}. In the next subsection, we will use other properties to construct low-degree approximations to the sampler and inverter.


\subsection{Approximating the Sampler-Inverter Pair via Truncation}\label{sec:truncated-insamp}

In this subsection,
we construct low-degree approximations to $\samp$ and $\insamp$
by truncating accepting paths of the automata representing local samplers.

Given $\*\rho\in\supp(\+P_{1:T})$ and the \emph{fixed} reference initial configuration $\sigma^\star$,
we first construct a low-depth circuit to compute the final configuration
$\samp(\*\rho)$.
Recall the definition of $\mathsf{AC}(d, s)$ beyond the Boolean domain in \pref{def:ac-over-q-n}.

\begin{lemma}
  \label{lem:app-samp-circuit}
  Suppose that $(G,\mu)$ satisfies the $\goodsampler$-good local sampler condition along $\*u$
  with respect to some simulator $((\+P_t,\psi_t))_{t\in[T]}$.
  For any positive integer $\ell_{\mathsf{samp}}$, there is a deterministic map
  $\appsamp:[Q]^T\to\set{\pm 1}^V$ such that
  \begin{equation*}
    \Pr[\*\rho\sim\+P_{1:T}]{\appsamp(\*\rho)\neq\samp(\*\rho)}
    \le n C \alpha^{\ell_{\mathsf{samp}}}.
  \end{equation*}
  Moreover, each coordinate of $\appsamp(\*\rho)$ and its Boolean complement can be computed by an
  $\mathsf{AC}(2, 8B\ell_{\mathsf{samp}}Q^{B\ell_{\mathsf{samp}}})$ circuit.
\end{lemma}
\begin{proof}
  Let $\set[t\in {[T]}, v\in V]{\+M^{\rightarrow}_{t,v}}$ be the local
  sampler given by \pref{cond:good-local-sampler}. Fix $v\in V$ and a mark
  sequence $\*\rho'\in[Q]^T$ for which the run of
  $\+M^{\rightarrow}_{T,v}$ on $(\*\rho',\sigma^\star)$ terminates. Let
  $(x_0,x_1,\ldots,x_m)$ be its accepting path, where
  $x_0=\texttt{init}$ and $x_m\in\bot$. 
  We associate to this path the certificate $\mathsf{Cert}^{v}_{\*\rho'}:[Q]^T\to\set{0,1}$ defined by
  \begin{equation}
    \label{eq:forward-certificate}
    \mathsf{Cert}^{v}_{\*\rho'}(\*\rho) \defeq
    \bigwedge_{j=0}^{m-1}
    \bigwedge_{s\in \posseq(x_j)}
    \idc[\rho_s = \rho'_s]
    .
  \end{equation}
  Note that by the local dependency condition, each certificate can be realized as a depth-$1$ circuit with the size at most $\sum_{j = 0}^{m - 1}|\posseq(x_j)| + 1\le mB + 1$. Here, the extra $1$ is due to the final AND gate.
  Also note that the certificate contains only mark coordinates, because the initial configuration is fixed to $\sigma^\star$.
  Consider the run of $\+M^{\rightarrow}_{T,v}$ on $(\*\rho,\sigma^\star)$ for any other mark sequence $\*\rho$.
  If
  $\mathsf{Cert}^{v}_{\*\rho'}(\*\rho)=1$, then by a simple induction from the initial state $x_0=\texttt{init}$, the run of $\+M^{\rightarrow}_{T,v}$ on $(\*\rho,\sigma^\star)$ follows the same path $(x_0,x_1,\ldots,x_m)$, reaches the same absorbing state, and generates the same output as $\+M^{\rightarrow}_{T,v}(\*\rho',\sigma^\star)$.

  Now fix $\ell_{\mathsf{samp}}>0$. 
  Define the following set of certificate functions:
  \begin{equation*}
    \+H_v = \set{\mathsf{Cert}^{v}_{\*\rho'} \mid \*\rho' \in [Q]^T, \text{ accepting path of } \+M^{\rightarrow}_{T,v}(\*\rho',\sigma^\star) \text{ has length } \leq \ell_{\mathsf{samp}} \text{ and ends at } \bot_+ }.
  \end{equation*}
  Identifying true with $+1$ and
  false with $-1$, define the $v$-th output coordinate by
  \begin{equation*}
  \forall \*\rho \in [Q]^T, \quad \appsamp_v(\*\rho)
    \defeq
    \bigvee_{H\in\+H_v} H(\*\rho),
  \end{equation*}
  where $\appsamp_v: [Q]^T \to \set{\pm 1}$ is a Boolean function.
  Next, we consider the circuit realization of $\appsamp_v$. For the fixed initial configuration $\sigma^\star$, an accepting path is determined by mark queries made along the run. At a state $x_i$, the next state depends only on the coordinates in $\posseq(x_i)$, since all configuration queries are evaluated on the fixed configuration $\sigma^\star$. By the local-dependency condition, there are at most $Q^{|\posseq(x_i)|}\le Q^B$ possible transitions at this step. Hence the number of possible length-$m$ paths, and therefore the number of certificates arising from accepting runs of length $m$, is at most $Q^{Bm}$. Consequently, $|\+H_v| \le \sum_{m=0}^{\ell_{\mathsf{samp}}} Q^{Bm} \le 2Q^{B\ell_{\mathsf{samp}}}$, where the last inequality is due to $Q^B\ge2$. 
  Each such certificate is a conjunction of at most $B\ell_{\mathsf{samp}}$ single-coordinate predicates, and can therefore be implemented by a depth-$1$ circuit of size at most $B\ell_{\mathsf{samp}}+1$. Taking the OR over all certificates gives a depth-$2$ circuit for $\appsamp_v$ of size at most $1+2Q^{B\ell_{\mathsf{samp}}}(B\ell_{\mathsf{samp}}+1)\le 8B\ell_{\mathsf{samp}}Q^{B\ell_{\mathsf{samp}}}$.
  The Boolean complement is computed without internal NOT gates by
  \begin{equation*}
    \neg\appsamp_v(\*\rho)=\bigwedge_{H\in\+H_v}\neg H(\*\rho).
  \end{equation*}
  Here each $\neg H$ is an OR of negated single-coordinate predicates, so negations remain at the input level and the resulting CNF has the same depth and size bound.
  Set the output of all coordinates by
  \begin{equation}
    \label{eq:appsamp-definition}
 \forall \*\rho \in [Q]^T, \quad \appsamp(\*\rho)
    =
    \tuple{\appsamp_v(\*\rho)}_{v\in V}.
  \end{equation}
  Hence, each coordinate of $\appsamp$ can be computed by an $\mathsf{AC}(2, 8B\ell_{\mathsf{samp}}Q^{B\ell_{\mathsf{samp}}})$ circuit.

  It remains to bound the error. Fix $v\in V$. If the run of
  $\+M^{\rightarrow}_{T,v}$ on $(\*\rho,\sigma^\star)$ terminates within
  $\ell_{\mathsf{samp}}$ steps at $\bot_+$, then its certificate belongs to
  $\+H_v$, and so $\appsamp_v(\*\rho)=+1$. If it terminates within
  $\ell_{\mathsf{samp}}$ steps at $\bot_-$, then no certificate in $\+H_v$ can
  be satisfied by $\*\rho$, because such a certificate would force the
  deterministic run on $(\*\rho,\sigma^\star)$ to follow a positive accepting
  path. Thus, on the event
  $L_{\mathrm{stop}}(\*\rho,\sigma^\star)\le\ell_{\mathsf{samp}}$,
  $\appsamp_v(\*\rho)$ agrees with the output of
  $\+M^{\rightarrow}_{T,v}$ on $(\*\rho,\sigma^\star)$, which equals
  $\samp(\*\rho)(v)$ by the definition of local sampler (\Cref{def:local-sampler}). Using the
  exponential-decay condition,
  \begin{equation*}
    \Pr[\*\rho\sim \+P_{1:T}]
    {\appsamp_v(\*\rho)\neq\samp(\*\rho)(v)}
    \le
    \Pr[\*\rho\sim \+P_{1:T}]{
      L_{\mathrm{stop}}(\*\rho,\sigma^\star)>\ell_{\mathsf{samp}}
    }
    \le C \alpha^{\ell_{\mathsf{samp}}}.
  \end{equation*}
  Taking a union bound over $v\in V$ gives the desired bound.
\end{proof}

Next, we approximate the one-hot version $\insamp^{\mathsf{oh}}$ of
$\insamp$ (\pref{alg:abstract-invsample}). We construct a randomized map
$\appinsamp$ with the same output format as $\insamp^{\mathsf{oh}}$: on input
a configuration $y\in\supp(\mu)$ and randomness $\*r$, it outputs a one-hot
string in $\set{0,1}^{[T]\times[Q]}$ in the sense of
\pref{def:one-hot-encoding}, with coordinates indexed directly by $[Q]$.
For $t\in[T]$ and
$j\in[Q]$, let $\appinsamp_{t,j}(y,\*r)\in\set{0,1}$ denote the
$(t,j)$-th coordinate, intended to approximate
$\insamp_{t,j}(y,\*r)$; the output of $\appinsamp$ is the concatenation
${(\appinsamp_{t,j})}_{t\in[T],\,j\in[Q]}$.
The next lemma gives such a low-degree approximation to
$\insamp^{\mathsf{oh}}$: after fixing the randomness $\*r$, each coordinate of
$\appinsamp$ is a low-degree polynomial in the input configuration.
\begin{lemma}
  \label{lem:app-invsamp}
  Suppose that $(G,\mu)$ satisfies the $\goodsampler$-good local sampler
  condition along $\*u^\leftarrow$ with respect to some inverse simulator $((\+P_{T-s+1},\psi_{T-s+1}))_{s\in[T]}$.
  For any positive integer $\ell_{\mathsf{inv}}$, there is a randomized
  procedure $\appinsamp$ that, on input $y\in\supp(\mu)$ and randomness
  $\*r$, outputs an element of $\set{0,1}^{[T]\times[Q]}$ encoding a mark sequence in $\supp(\+P_{1:T})$, and satisfies
  \begin{equation*}
    \forall y\in\supp(\mu),\qquad
    \Pr[\*r]{\insamp^{\mathsf{oh}}(y,\*r)\neq \appinsamp(y,\*r)}
    \le
    T C\alpha^{\ell_{\mathsf{inv}}}.
  \end{equation*}
  Moreover, if we fix the randomness $\*r$, for every $i\in[T]$ and $j\in[Q]$, the Boolean function $\appinsamp_{i,j}(\cdot,\*r): \supp(\mu) \to \{0,1\}$ has degree at most
  $(\Delta+1)B\ell_{\mathsf{inv}}$.
\end{lemma}

\begin{proof}
  Although $\insamp$ in \Cref{alg:abstract-invsample} was introduced using auxiliary random bit strings, in this
  proof we work directly with the finite-support random variables generated by
  those bits. This is only a reparameterization of the same randomness: since
  all distributions below are supported on finite sets, each draw can be
  generated exactly from an independent infinite unbiased bit string.
  Write $\*r=(\*r^{\mathsf{seq}},\*r^{\mathsf{con}})$, with
  $\*r^{\mathsf{con}}=(\*r_t^{\mathsf{con}})_{t\in[T]}$. 
  We may view $\*r^{\mathsf{seq}}$ as a sequence $\*r^{\mathsf{seq}} = (\*r^{\mathsf{seq}}_t)_{t\in[T]}$. Each $\*r^{\mathsf{seq}}_t \sim \+P_t$ is an independent sample.
  The variable
  $\*r^{\mathsf{seq}}$ is used to generate the reverse mark sequence
  $\*\eta$ in \pref{line:inv-sample-eta} of $\insamp$. For each $s\in[T]$, set $\eta_s=\*r^{\mathsf{seq}}_{T-s+1}$.
  For each
  $t\in[T]$, the conditional-resampling randomness
  $\*r_t^{\mathsf{con}}$ is a finite table indexed by pairs
  $(\tau,c)\in\{\pm 1\}^{N(u_t)}\times\{\pm 1\}$, and each entry is
  valued in $[Q]$. For simplicity, we write
  $(\*r_t^{\mathsf{con}})_{\tau,c}$ as
  $\*r^{\mathsf{con}}_{t,\tau,c}$, so that
  $\*r^{\mathsf{con}}$ may be viewed as indexed by triples $(t,\tau,c)$.

  If the conditioning event $\psi_t(\rho,\tau)=c$ has positive
  probability, the entry $\*r^{\mathsf{con}}_{t,\tau,c}$ has distribution
  \begin{equation}
    \label{eq:rcon-distribution}
    \forall j\in[Q],\qquad
    \Pr{\*r^{\mathsf{con}}_{t,\tau,c}=j}
    =
    \Pr[\rho\sim\+P_t]{\rho=j\mid\psi_t(\rho,\tau)=c}.
  \end{equation}
  If the conditioning event has probability zero, set
  $\*r^{\mathsf{con}}_{t,\tau,c}=a$ with probability one, where $a$ is an arbitrary element in the support of $\+P_t$. All entries of
  the table are sampled independently. In \pref{line:mark-sampling}, when the local data are
  $(\tau,c)$, the conditional sampling call is implemented by outputting the
  pre-sampled table entry $\*r^{\mathsf{con}}_{t,\tau,c}$. Hence the table has
  exactly the conditional law required by $\insamp$, while its distribution is
  chosen in advance and does not depend on the input configuration $y$.

  Let $\set[s\in{[T]},\,w\in V]{\+M_{s,w}^{\leftarrow}}$ be the reverse local
  sampler given by \pref{cond:good-local-sampler}.
  For $v\in V$, $t\in [T]$,
  $y'\in\supp(\mu)$, and $\*\eta'\in[Q]^T$, consider the run of
  $\+M^{\leftarrow}_{t, v}(\*\eta',y')$ that resolves
  $\Psi_{t}^{\leftarrow}(\*\eta',y')(v)$.
  Let $(x_0,x_1,\ldots,x_m)$ be its accepting path.
  Define the certificate for it as
  \begin{equation}
    \label{eq:reverse-certificate}
    \mathsf{Cert}^{t,v}_{\*\eta',y'}(\*\eta,y)\defeq
    \prod_{j=0}^{m-1}
    \tuple{
      \prod_{s\in \posseq(x_j)}
      \idc[\eta_s = \eta'_s]
      \cdot
      \prod_{i\in \poscon(x_j)}
      \idc[y_i=y'_i]
    }.
  \end{equation}
  The reverse certificate in \pref{eq:reverse-certificate} differs from the
  forward certificate in \pref{eq:forward-certificate}: there the initial
  configuration is fixed to $\sigma^\star$, whereas here the terminal
  configuration is the input variable $y$, so the certificate must also record
  the queried coordinates of this configuration through the equalities
  $y_i=y'_i$.
  If another mark sequence $\*\eta$ and configuration $y$ satisfy
  $\mathsf{Cert}^{t,v}_{\*\eta',y'}(\*\eta,y)=1$, then by a simple induction from the initial state $x_0=\texttt{init}$, the run of
  $\+M^{\leftarrow}_{t, v}(\*\eta,y)$ follows the same path
  $(x_0,x_1,\ldots,x_m)$ and therefore produces the same output as
  $\+M^{\leftarrow}_{t, v}(\*\eta',y')$.
  We remark that in terms of \Cref{alg:abstract-invsample}, the output of the automaton is one coordinate of
  the reverse trajectory: it equals $\Psi_t^\leftarrow(\*\eta,y)(v)$, i.e.,
  the $v$-th spin of $\sigma_{T-t}$ in \pref{line:reverse-trajectory} of \Cref{alg:abstract-invsample}.

  Define the following set of certificate functions:
  \begin{equation*}
    \begin{aligned}
      \+H_{t,v}^{+}
      =
      \Bigl\{
      \mathsf{Cert}^{t,v}_{\*\eta',y'}
      \,\Big|\,
      &(\*\eta',y')\in[Q]^T\times\supp(\mu),\\
      &\text{accepting path of }
      \+M^{\leftarrow}_{t,v}(\*\eta',y')
      \text{ has length } \le \ell_{\mathsf{inv}}
      \text{ and ends at } \bot_+
      \Bigr\}.
    \end{aligned}
  \end{equation*}
  Since $\+H_{t,v}^{+}$ is a set, multiple pairs that produce the same
  certificate function contribute only one element.
  We define the truncated local approximation by
  \begin{equation*}
    \widehat{\Psi}_t^\leftarrow(\*\eta,y)(v)
    \defeq
    2\sum_{H\in \+H_{t,v}^{+}} H(\*\eta,y)-1
    =
    \begin{cases}
      1 & \text{if $\sum_{H\in \+H_{t,v}^{+}} H(\*\eta,y)=1$} \\
      -1 & \text{if $\sum_{H\in \+H_{t,v}^{+}} H(\*\eta,y)=0$.}
    \end{cases}
  \end{equation*}
  The sum $\sum_{H\in \+H_{t,v}^{+}} H(\*\eta,y)$ is always either $0$ or
  $1$. Indeed, each certificate records all coordinates queried along its
  path, and therefore satisfying the certificate forces the run on
  $(\*\eta,y)$ to follow that path. If two positive certificates in
  $\+H_{t,v}^{+}$ were both satisfied by the same input, then the unique
  deterministic run on this input would be forced to follow both associated
  accepting paths, so these paths must coincide. Along this common path, the
  same input gives the same values to every queried coordinate, and hence the
  two certificate functions are identical. Since $\+H_{t,v}^{+}$ is a set of
  certificate functions (no duplicate certificates), they must contribute the same element.
  Therefore,
  $\widehat{\Psi}_t^\leftarrow(\*\eta,y)(v)\in\{\pm 1\}$ for every input.

  We reconstruct an approximate trajectory by setting $\widehat\sigma_T=y$ and,
  for $t=T,T-1,\ldots,1$, setting
  \begin{equation*}
    \widehat\sigma_{t-1}(w)=\widehat\sigma_t(w)\quad(w\neq u_t),
    \qquad
    \widehat\sigma_{t-1}(u_t)=\widehat{\Psi}_{T-t+1}^\leftarrow(\*\eta,y)(u_t).
  \end{equation*}
  Next, we define the Boolean function $\appinsamp_{t,j}$ in the lemma. The idea is to use $\widehat{\sigma}_t$ to approximate the exact trajectory.
  To make $\appinsamp_{t,j}$ an explicit polynomial, we enumerate all possible
  values of the local configuration. Using the standard indicator-polynomial
  convention, define
  \begin{equation}
    \label{eq:appinsamp-polynomial}
    \begin{aligned}
      \appinsamp_{t,j}(y,\*r)
      \defeq{}&
      \sum_{\tau\in\{\pm 1\}^{N(u_t)}}\sum_{c\in\{\pm 1\}}
      \idc\left[\*r^{\mathsf{con}}_{t,\tau,c}=j\right]
      \prod_{v\in N(u_t)}\frac{1+\tau_v\widehat\sigma_{t-1}(v)}{2}
      \cdot\frac{1+c\widehat\sigma_t(u_t)}{2}.
    \end{aligned}
  \end{equation}
  Since every $\widehat\sigma_s(v)$ takes values in $\{\pm 1\}$, the factor
  $(1+b\widehat\sigma_s(v))/2$ is exactly the indicator of
  $\widehat\sigma_s(v)=b$. Hence exactly one pair $(\tau,c)$ contributes to
  \pref{eq:appinsamp-polynomial}, and its coefficient is $\idc\left[\*r^{\mathsf{con}}_{t,\tau,c}=j\right]$. In particular, $\appinsamp(y,\*r)$ is a one-hot encoding of a mark sequence in $\supp(\+P_{1:T})$
  for every input.

  \paragraph{Degree analysis.}
  We now analyze the degree of each output coordinate of $\appinsamp$. In this
  paragraph, $\*r$ is fixed, and therefore the reverse mark sequence $\*\eta$
  and the conditional-resampling table are also fixed. Thus $\codeg_y$ denotes
  the coordinate degree from \pref{eq:coordinate-degree}, applied only to the
  input configuration $y$; in particular, the monomial
  $\prod_{i\in S}\idc[y_i=b_i]$ has $\codeg_y=|S|$.

  Consider a certificate
  $\mathsf{Cert}^{t,v}_{\*\eta',y'}(\*\eta,y)$ arising from an accepting path
  $(x_0,\ldots,x_m)$ with $m\le\ell_{\mathsf{inv}}$. After $\*\eta$ is fixed,
  all factors involving the mark sequence are constants. The only remaining
  $y$-dependent factors are the indicators $\idc[y_i=y'_i]$ for coordinates
  $i\in\poscon(x_j)$ along the path. Since each state queries at most $B$
  coordinates of the configuration, the whole path queries at most
  $B\ell_{\mathsf{inv}}$ such coordinates. Hence every certificate has
  $\codeg_y$ at most $B\ell_{\mathsf{inv}}$.
  The function $\widehat{\Psi}_t^\leftarrow(\*\eta,y)(v)$ is a linear
  combination of these certificates, together with a constant term. Taking
  such a sum does not increase coordinate degree, so
  $\widehat{\Psi}_t^\leftarrow(\*\eta,y)(v)$ also has $\codeg_y$ at most
  $B\ell_{\mathsf{inv}}$. Finally, each recursively defined value
  $\widehat\sigma_t(v)$ is either an original input spin $y_v$ or one of the
  local approximations $\widehat{\Psi}_s^\leftarrow(\*\eta,y)(v)$. Since
  $B\ell_{\mathsf{inv}}\ge1$, it follows that every $\widehat\sigma_t(v)$ has
  $\codeg_y$ at most $B\ell_{\mathsf{inv}}$.

  Fix $t\in[T]$ and $j\in [Q]$.
  After fixing $\*r$, the indicators involving $\*r^{\mathsf{con}}$ in \pref{eq:appinsamp-polynomial}
  are constants. Each summand is a product of at most $\abs{N(u_t)}+1\le\Delta+1$ polynomials of $\codeg_y$ at most $B\ell_{\mathsf{inv}}$. 
  Hence every summand, and therefore $\appinsamp_{t,j}(\cdot,\*r)$, has $\codeg_y$ at most $(\Delta+1)B\ell_{\mathsf{inv}}$, as claimed.

  \paragraph{Probability of disagreement.}
  We now bound the probability that $\appinsamp(y,\*r)$ differs from
  $\insamp^{\mathsf{oh}}(y,\*r)$. If every reverse update value
  $\widehat{\Psi}_{T-t+1}^\leftarrow(\*\eta,y)(u_t)$ equals the exact value
  $\Psi_{T-t+1}^\leftarrow(\*\eta,y)(u_t)=\sigma_{t-1}(u_t)$, then a backward induction from
  $\widehat\sigma_T=\sigma_T=y$ shows that
  $\widehat\sigma_t=\sigma_t$ for all $0\le t\le T$.  On this event, every
  conditional resampling step in $\appinsamp$ uses exactly the same local data as
  $\insamp$, and hence $\appinsamp$ outputs the one-hot encoding of the mark
  sequence produced by $\insamp$ under the same internal randomness. For each
  $t$, consider the run of
  $\+M^{\leftarrow}_{T-t+1,u_t}$ on $(\*\eta,y)$. If this run terminates
  within $\ell_{\mathsf{inv}}$ steps, then the certificate construction above
  gives
  $\widehat{\Psi}_{T-t+1}^\leftarrow(\*\eta,y)(u_t)
  =\Psi_{T-t+1}^\leftarrow(\*\eta,y)(u_t)$: positive accepting paths are
  detected by $\+H_{T-t+1,u_t}^{+}$, while negative accepting paths satisfy no
  positive certificate and hence receive the default value $-1$. Therefore the
  event
  $\widehat{\Psi}_{T-t+1}^\leftarrow(\*\eta,y)(u_t)
  \neq\Psi_{T-t+1}^\leftarrow(\*\eta,y)(u_t)$ can occur only if the
  corresponding accepting path is truncated before reaching an absorbing state.
  By the exponential decay property in \pref{cond:good-local-sampler},
  applied to the reverse local sampler and the law induced by $\*r^{\mathsf{seq}}$,
  this has
  probability at most $C\alpha^{\ell_{\mathsf{inv}}}$.  A union bound over
  $t\in[T]$ gives
  \begin{equation*}
    \Pr[\*r]{\insamp^{\mathsf{oh}}(y,\*r)\neq \appinsamp(y,\*r)}
    \le
    T C\alpha^{\ell_{\mathsf{inv}}}.\qedhere
  \end{equation*}
\end{proof}

\subsection{Putting Things Together}\label{sec:put-together}

We are now putting all the pieces together to prove \Cref{thm:good-sampler-polynomial-approximation}. We first prove the following lemma, which says if $f \circ \appsamp$ can be approximated by a low degree polynomial $g$ under the product distribution $\+P_{1:T}$, then $f$ can be approximated by some low degree polynomial (defined by $\appinsamp$) under the Gibbs distribution $\mu$. Recall that indicator polynomials are defined in \pref{eq:indicator-polynomial}.

\begin{lemma}
  \label{lem:grand-unified-theory}
  Fix a time horizon $T\ge n$ and a product distribution $\+P_{1:T}$ over $[Q]^T$.
  Suppose that, with mark domain $[Q]$, $(G,\mu)$ satisfies the bi-directional $\goodsampler$-good local sampler condition along $\*u$.
  Let $\ell_{\mathsf{samp}}$ and $\ell_{\mathsf{inv}}$ be the truncation sizes of $\appsamp$ and $\appinsamp$, respectively.
  Let $f:\set{\pm 1}^V\to \set{\pm 1}$ be a Boolean function.
  Suppose that there is an indicator polynomial $g:[Q]^T\to \^R$ such that
  \begin{equation}
    \label{eq:appsamp-polynomial-approximation}
    \E[\*\rho\sim \+P_{1:T}]{\abs{f\circ \appsamp(\*\rho)-g(\*\rho)}^p}
    \le \varepsilon.
  \end{equation}
  Let $M=(1+\norm{g}_{L^\infty(\+P_{1:T})})^p$.
  Then there exists $\*r^\star$ such that
  \begin{equation*}
    \E[y\sim\mu]{\abs{f(y)-g\circ\appinsamp(y,\*r^\star)}^p}
    \le 2^{p-1} \cdot \varepsilon + 2^{2p-1} \cdot nC\alpha^{\ell_\mathsf{samp}} + M \cdot \tp{nC\alpha^{\lfloor\frac{T - n}{n}\rfloor} + T C\alpha^{\ell_{\mathsf{inv}}}}.
  \end{equation*}
  Moreover, if $\codeg(g)\le\phi$, then
  $\deg(g\circ\appinsamp(\cdot,\*r^\star))\le
  \phi B\ell_{\mathsf{inv}}(\Delta+1)$.
\end{lemma}

\begin{proof}
  Observe that the bi-directional $\goodsampler$-good local sampler condition along $\*u$ satisfies the hypotheses of \pref{lem:good-local-sampler-determining}, \pref{lem:app-samp-circuit}, and \pref{lem:app-invsamp}, allowing us to invoke these lemmas directly in the subsequent proof.
  
  We first transfer the approximation guarantee from $f\circ\appsamp$ to $f\circ\samp$ using the triangle inequality and the inequality ${(a+b)}^p\le 2^{p-1}(a^p+b^p)$ for non-negative $a$ and $b$,
  \begin{align*}
    &\E[\*\rho\sim\+P_{1:T}]{\abs{f\circ \samp(\*\rho) - g(\*\rho)}^p}\\
    \le & \;\; 2^{p - 1} \cdot \tp{\E[\*\rho\sim\+P_{1:T}]{\abs{f\circ \samp(\*\rho) - f\circ \appsamp(\*\rho)}^p} + \E[\*\rho\sim\+P_{1:T}]{\abs{f\circ \appsamp(\*\rho) - g(\*\rho)}^p}} \\
    \le & \;\; 2^{2p - 1} \cdot n C \alpha^{\ell_\mathsf{samp}} + 2^{p - 1} \cdot \varepsilon,
  \end{align*}
  where the last inequality uses \pref{lem:app-samp-circuit} and \pref{eq:appsamp-polynomial-approximation}.

  The preceding estimate is now in the form required by the sampler–inverter comparison lemma.
  Applying \pref{lem:sample-inv-comparison} with $\epsDet=nC\alpha^{\lfloor (T-n)/n \rfloor}$
  as established in \pref{lem:good-local-sampler-determining} gives
  \begin{equation*}
    \E[y\sim\mu,\*r]{\abs{f(y)-g\circ\insamp^{\mathsf{oh}}(y,\*r)}^p} \le 2^{p-1} \cdot \varepsilon + 2^{2p-1} \cdot nC\alpha^{\ell_\mathsf{samp}} + M \cdot nC\alpha^{\lfloor(T-n)/n\rfloor}.
  \end{equation*}

  It remains to replace $\insamp^{\mathsf{oh}}$ by $\appinsamp$.
  For each fixed $y\in\supp(\mu)$, \pref{lem:app-invsamp} gives
  a uniform bound
  \begin{equation*}
    \Pr[\*r]{\insamp^{\mathsf{oh}}(y,\*r)\neq \appinsamp(y,\*r)}
    \le T C\alpha^{\ell_{\mathsf{inv}}}.
  \end{equation*}
  Hence the same bound holds after averaging over $y\sim\mu$.
  As above, on the event $\insamp^{\mathsf{oh}}(y,\*r)=\appinsamp(y,\*r)$,
  replacing the former by
  the latter does not change the loss. On the complementary event, both are
  valid one-hot encodings of mark sequences in $\supp(\+P_{1:T})$, so the loss is at most $M$. Therefore
  \begin{equation*}
    \begin{aligned}
      \E[y\sim\mu,\*r]{\abs{f(y)-g\circ\appinsamp(y,\*r)}^p} &\leq \; \E[y\sim\mu,\*r]{\abs{f(y)-g\circ\insamp^{\mathsf{oh}}(y,\*r)}^p} +M \cdot T C\alpha^{\ell_{\mathsf{inv}}} \\
      &\leq \;\; 2^{p-1} \cdot \varepsilon + 2^{2p-1} \cdot nC\alpha^{\ell_\mathsf{samp}} + M \cdot \tp{nC\alpha^{\lfloor\frac{T-n}{n}\rfloor} + T C\alpha^{\ell_{\mathsf{inv}}}}.
    \end{aligned}
  \end{equation*}
  Since the left-hand side is the average, over $\*r$, of
  $\E[y\sim\mu]{\abs{f(y)-g\circ\appinsamp(y,\*r)}^p}$, there exists a fixed
  $\*r^\star$ satisfying the same upper bound.

  Finally, the degree claim holds for the same $\*r^\star$. By
  \pref{lem:app-invsamp}, after fixing $\*r^\star$, every coordinate
  $\appinsamp_{i,j}(\cdot,\*r^\star)$ of
  $\appinsamp(\cdot,\*r^\star)$ is a polynomial in the input configuration
  $y$ of degree at most $(\Delta+1)B\ell_{\mathsf{inv}}$.
  Recall that $g:[Q]^T\to \mathbb R$ is an indicator polynomial with
  $\codeg(g)\le\phi$.  Each monomial of $g$ involves at most $\phi$
  indicators of the form $\idc[\rho_i=j]$. Under the composition with
  $\appinsamp(\cdot,\*r^\star)$, the indicator $\idc[\rho_i=j]$ is replaced by
  $\appinsamp_{i,j}(\cdot,\*r^\star):\{\pm 1\}^V \to \{0,1\}$. Hence each substituted monomial has
  degree at most $\phi B\ell_{\mathsf{inv}}(\Delta+1)$, and taking linear
  combinations does not increase degree. Therefore
  $
    \deg(g\circ\appinsamp(\cdot,\*r^\star))
    \le \phi B\ell_{\mathsf{inv}}(\Delta+1).
  $
\end{proof}

Finally, we are ready to prove \Cref{thm:good-sampler-polynomial-approximation}.

\begin{proof}[Proof of \Cref{thm:good-sampler-polynomial-approximation}]
  Let $T\ge n$ be a time horizon to be determined later.
  Set $\eps_{\poly}=\eps/8$ and $\ell_{\mathsf{samp}}=O(\log (n / \eps))$, with the hidden constant depending only on $C$ and $\alpha$, so that $8nC\alpha^{\ell_{\mathsf{samp}}}\le\eps/4$.
  Set $d'=d+2$ and $s'= O(s n B\ell_{\mathsf{samp}}Q^{B\ell_{\mathsf{samp}}})$.
  By \pref{lem:app-samp-circuit}, each positive literal $x_v$ in the circuit for $f$ is replaced by the circuit for $\appsamp_v$, and each negative literal $\neg x_v$ is replaced by the circuit for $\neg\appsamp_v$; the factor $16nB\ell_{\mathsf{samp}}Q^{B\ell_{\mathsf{samp}}}$ accounts for these two possible replacements at each vertex.
  Thus all negations remain at input predicates, and the function $h(\*\rho)=f(\appsamp(\*\rho))$ is computed over $[Q]^T$ by an $\mathsf{AC}(d',s')$ circuit and
  \begin{equation*}
    \Pr[\*\rho\sim\+P_{1:T}]{\appsamp(\*\rho)\neq\samp(\*\rho)}
    \le nC\alpha^{\ell_{\mathsf{samp}}}.
  \end{equation*}
  We apply \pref{thm:low-degree-beyond-boolean} to the product space $[Q]^T$ equipped with $\+P_{1:T}$, using circuit parameters $(d',s')$ and accuracy $\eps_{\poly}$. Thus there is an indicator polynomial $g:[Q]^T\to\mathbb R$ such that
  \begin{equation*}
    \E[\*\rho\sim\+P_{1:T}]{(h(\*\rho)-g(\*\rho))^2}\le\eps_{\poly},
    \quad \codeg(g)\le \phi,
    \quad \norm{g}_{L^\infty(\+P_{1:T})}\le \Phi,
  \end{equation*}
  where $\phi=\min\{T,\lceil16C_{d'}\log^{d'-1}(2s')\log(2/\eps_{\poly})\rceil\} = \lceil16C_{d'}\log^{d'-1}(2s')\log(2/\eps_{\poly})\rceil$ (because $T \ge n$ and $n$ is sufficiently large) and $\Phi=(2\e T/\phi)^\phi$.
  By \pref{lem:grand-unified-theory} with moment parameter $p = 2$, there exists a realization $\*r^\star$ such that
  \begin{equation*}
    \E[y\sim\mu]{(f(y)-g(\appinsamp(y,\*r^\star)))^2}
    \le
    2\eps_{\poly}+8nC\alpha^{\ell_{\mathsf{samp}}}
    +(1+\Phi)^2\left(nC\alpha^{\lfloor \frac{T - n}{n} \rfloor}+TC\alpha^{\ell_{\mathsf{inv}}}\right).
  \end{equation*}
  Since $\ell_{\mathsf{samp}}=O(\log(n / \eps))$ and $s=n^c$, we have $\log s'=O(\log(n / \eps))$, where the hidden constant depends only on $B,Q$ and $c$.
  Hence $\phi=\log^{O(d)}(n/\eps)$ and $\log\Phi=\log^{O(d)}(n/\eps)\log \frac{T}{\phi}$.
  We may choose $T = n\log^{O(d)}(n/\eps)$ with sufficiently large hidden constants depending only on $B,C,\alpha,Q,\Delta$ and $c$ hence (1) $\lfloor \frac{T-n}{n} \rfloor=\log^{O(d)}(n/\eps)$ and (2) $(1+\Phi)^2 nC\alpha^{\lfloor \frac{T-n}{n}\rfloor}\le\eps/4$.
  Finally choose $\ell_{\mathsf{inv}}=\log^{O(d)}(n/\eps)$ so that $(1+\Phi)^2TC\alpha^{\ell_{\mathsf{inv}}}\le\eps/4$.
  For this choice, we define the polynomial $p(y)=g(\appinsamp(y,\*r^\star))$ with the input $y \in \{\pm 1\}^V$ and the output $p(y)\in\mathbb R$.
  Then \pref{lem:grand-unified-theory} gives the low-degree approximation guarantee $\E[y\sim\mu]{(f(y)-p(y))^2}\le \eps$ and the degree bound
  \begin{equation*}
    \deg(p)\le \phi B\ell_{\mathsf{inv}}(\Delta+1)
    \le \log^{O(d)}(n/\eps). \qedhere
  \end{equation*}
\end{proof}

\section{Applications}
\label{sec:applications}

We now instantiate the framework for two classes of two-spin systems.
Recall that $\*u=(u_1,\ldots,u_T)$ is the cyclic update sequence fixed in
\pref{sec:framework}, with $u_t=v_{((t-1)\bmod n)+1}$, and
$\*u^\leftarrow=(u_T,\ldots,u_1)$ is its reversal.
The goal is to verify the bi-directional good local sampler condition; the
learning theorems then follow from \pref{thm:good-sampler-polynomial-approximation}.

\subsection{Hard Constraint Graphical Model}

Recall that the hardcore model with fugacity $\lambda > 0$ is defined on a graph $G = (V, E)$ with maximum degree $\Delta$. The distribution $\mu$ is supported on independent sets with weight proportional to $\lambda^{|I|}$, where $\abs{I}$ is the size of the independent set $I$.

\begin{lemma}[Bi-directional good local sampler for the hard-core model]
  \label{lem:hardcore-good-local-sampler}
  Fix $\Delta\ge2$ and $\eta\in(0,1)$.
  Let $G=(V,E)$ be a graph of maximum degree at most $\Delta$, and let $\mu$ be the hard-core distribution on $G$ with fugacity $\lambda<(1-\eta)/(\Delta-1)$.
  Then, for every time horizon $T\ge n$, there is an update simulator $((\+P_t,\psi_t))_{t\in[T]}$ with mark domain $[Q]$ such that $(G,\mu)$ satisfies the bi-directional $\goodsampler$-good local sampler condition along $\*u$, where
  \[
    Q=2, \quad B=1,\quad C=\exp(\eta^2/(4\Delta^2)),\quad \alpha=\exp(-\eta^2/(80\Delta^3)).
  \]
\end{lemma}

We first derive the learning theorem from the sampler lemma; after that we verify the lemma by describing the local sampler and its automaton implementation.

\begin{proof}[Proof of \pref{thm:intro-hardcore-learning}]
  By \pref{lem:hardcore-good-local-sampler}, the hard-core model satisfies the  bi-directional good local sampler condition with constants depending only on $\Delta$ and $\eta$.
  Applying \pref{thm:good-sampler-polynomial-approximation} with accuracy $\varepsilon/2$, every $f\in\mathsf{AC}(d,n^c)$ has an $L^2(\mu)$ approximating polynomial of degree $\log^{O(d)}(n/\varepsilon)$, where the hidden constant in $O(d)$ depends only on $c,\Delta,\eta$.
  The low-degree regression theorem recalled in \pref{sec:low-degree-learning} then gives the stated learner with sample complexity and running time $n^{\log^{O(d)}(n/\varepsilon)}$, error at most $\varepsilon$, and success probability at least $0.9$.
\end{proof}

We now turn to the proof of \pref{lem:hardcore-good-local-sampler}.
The construction has two layers: first a recursive local resolver for the hard-core heat-bath update, and then a finite automaton that implements this recursion.
Here, we set $Q = 2$ with $[Q] = \{0, \bot\}$.
This is a two-symbol mark alphabet: the symbols $0$ and $\bot$ are labels for
the two possible update instructions, not numerical values.
The distribution $\+P_{1:T}$ is a product distribution over $\{0, \bot\}^T$,
where each marginal $\+P_t$ satisfies
$\+P_t(0)=\frac{1}{1+\lambda}$ and
$\+P_t(\bot)=\frac{\lambda}{1+\lambda}$.
Let $\psi_t$ be the deterministic heat-bath rule induced by these marks: mark
$0$ outputs $-1$, while mark $\bot$ outputs $+1$ iff all neighbor spins are
$-1$. Thus $((\+P_t,\psi_t))_{t\in[T]}$ is the forward update simulator.
The reversal of this simulator is
$((\+P_{T-s+1},\psi_{T-s+1}))_{s\in[T]}$ along $\*u^\leftarrow$.

For simplicity, we write $i_t$ for the vertex updated at time $t$.
For convenience, we write $\texttt{Last}(v, t) = \max \set{s \mid s = 0 \vee (s \le t \wedge i_s = v)}$ for the last time no later than $t$ that $v$ is updated, where we set $\texttt{Last}(v, t) = 0$ if there is no such time.
We regard $V$ as equipped with an arbitrary fixed order; all uses of
``smallest'' below refer to this order.

\paragraph{Local sampler.}
Before spelling out the automaton, we describe the recursive procedure it simulates.
Fix an initial configuration $\sigma_0\in\supp(\mu)$ and a mark sequence $\*\rho\in\{0,\bot\}^T$.
For a vertex $v$ and a time $0\le t\le T$, the resolver $\resolve(v,t)$ returns the value of $\sigma_t(v)$ in the Glauber trajectory driven by $\*\rho$.
If $v$ is not updated at time $t$, the procedure first moves back to the last update time of $v$.

\begin{algorithm}[htbp]
  \caption{Hard-core local resolver $\resolve(v,t)$}
  \label{alg:hardcore-local-resolver}
  \KwIn{a vertex $v$ and a time $0\le t\le T$}
  \KwOut{the resolved spin $\sigma_t(v)\in\{\pm 1\}$}

  Set $s\gets \texttt{Last}(v,t)$\;
  \If{$s=0$}{
    \Return{$\sigma_0(v)$}\;
  }

  Query the mark $\rho_s$\;
  \If{$\rho_s=0$}{
    \Return{$-1$}\;
  }

  \For{\textnormal{each neighbor }$w\in N(v)$ \textnormal{in the fixed order}}{
    Set $r\gets \texttt{Last}(w,s)$\;
    \If{$\resolve(w,r)=+1$}{
      \Return{$-1$}\;
    }
  }

  \Return{$+1$}\;
\end{algorithm}

This is exactly the heat-bath update written in a local form.
The mark $\rho_s=0$ immediately outputs the unoccupied spin $-1$.
When $\rho_s=\bot$, the update attempts to occupy $v$, and this succeeds exactly when every neighbor is unoccupied at the previous relevant time.
Thus the recursive calls are made only for the neighbor values needed to decide the hard constraint.

\paragraph{Automaton representation.}
We now implement the resolver above by a finite-state automaton whose state contains a stack.
The input to the automaton is the same initial configuration $\sigma_0\in\supp(\mu)$ and random sequence $\*\rho\sim \+P_{1:T}$.
The initial configuration $\sigma_0$ and the random sequence $\*\rho$ are shared universally across the automata and we drop them from the notation in the state space for simplicity.
Note that we query the initial configuration and the random sequence in the transition rules.

Let $S$ denote the stack maintaining the state of unresolved calls, with each element taking the form $((u, t, \!X), \texttt{state})$.
The tuple $(u, t, \!X)$ records the vertex $u$ and the time $t$ of the marginal to be resolved, and $\!X$ is a set of configurations with time stamps of the form $(v, s, \sgn)$, denoting that the value of $v$ at time $s$ is $\sgn$. The parameter $\texttt{state}$ indicates the current execution state of the call.
The cache $\!X$ is threaded through the depth-first execution: a child frame is initialized with the current cache of its parent, and when the child returns, the child's final cache is merged back into the parent.
Thus $\!X$ records all resolved triples known to the current branch of the execution.

We categorize the possible stack elements into four types as follows.
\begin{itemize}
  \item \textbf{Init} $((u,t,\!X),\texttt{init})$: The initialization element indicates the entry point of an unresolved call. Note that each vertex is updated only at times $t$ with $i_t = u$, hence we make the convention that the input time for resolving $\sigma_t(u)$ satisfies $t = 0$ or $i_t = u$. Otherwise, we simply modify the input time to the last update time $\texttt{Last}(u, t)$.
  \item \textbf{Inspect} $((u,t,\!X),\texttt{inspect})$: The inspection element indicates that the algorithm is ready to inspect the neighbor vertices at time $t$.
  \item \textbf{Hold} $((u,t,\!X),(\texttt{hold},w,s))$: The holding element indicates that the algorithm is waiting for a recursive call resolving $\sigma_s(w)$.
  \item \textbf{Halt} $((u,t,\!X),(\texttt{halt},\sgn))$: The halting element indicates a termination of the call resolving $\sigma_t(u)$ with output $\sgn$.
\end{itemize}

Initially, the stack $S$ contains a single element $((v, \texttt{Last}(v,t), \emptyset), \texttt{init})$. In each execution step, the algorithm updates the stack $S$ according to its top element and, when returning from a recursive call, the element immediately below it.
Let $((u, t, \!X), \texttt{state})$ be the top element of the stack. We update the stack according to the following rules.
\begin{enumerate}
  \item If $\texttt{state} = \texttt{init}$, note that $t = 0$ or $i_t = u$. If $(u,t,\sgn)\in \!X$ for some $\sgn$, we update the top element to $((u,t,\!X),(\texttt{halt},\sgn))$. Otherwise, if $t = 0$, we perform the query to the initial configuration $\sigma_0$ at vertex $u$ and update the top element to $((u, t, \!X), (\texttt{halt}, \sigma_0(u)))$. Otherwise, we query the random sequence $\*\rho$ at time $t$. If $\rho_t = 0$, we update the top element to $((u, t, \!X), (\texttt{halt}, -1))$. If $\rho_t = \bot$, we update the top element to $((u, t, \!X), \texttt{inspect})$.
  \item If $\texttt{state} = \texttt{inspect}$, let $\tau^\star = \set{(x, \sgn): (x, \texttt{Last}(x, t), \sgn)\in \!X}$ be the current knowledge of vertices whose values at time $t$ have been resolved and let $\Lambda^*$ be the set of these vertices. 
  Check the following cases in order:
  \begin{enumerate}
    \item If $(x,+1)\in\tau^\star$ for some $x\in N(u)$, we update the top element to $((u,t,\!X),(\texttt{halt},-1))$.
    \item Else if $N(u) \subseteq \Lambda^*$, we update the top element to $((u, t, \!X), (\texttt{halt}, +1))$.
    \item Otherwise, we pick the smallest vertex $w\in N(u)\setminus \Lambda^*$ and let $s = \texttt{Last}(w, t)$. We update the top element to $((u, t, \!X), (\texttt{hold}, w, s))$ and push a new element $((w, s, \!X), \texttt{init})$ onto the stack.
  \end{enumerate}
  \item If $\texttt{state} = (\texttt{halt}, \sgn)$, write the top element as $((u^\circ,t^\circ,\!X^\circ),(\texttt{halt},\sgn))$ and pop it.
  \begin{enumerate}
    \item If the stack is empty, we return $\sgn$ as the output and the whole automaton reaches the termination state $\bot_\sgn$.
    \item Otherwise, the current top element has the form $((u^\star, t^\star, \!X^\star), (\texttt{hold}, w^\circ, s^\circ))$. We then set $\!X'=\!X^\star\cup\!X^\circ \cup \{(w^\circ, s^\circ, \sgn)\}$. If $\sgn = +1$, we update the top element to $((u^\star, t^\star, \!X'), (\texttt{halt}, -1))$. Otherwise, we update the top element to $((u^\star, t^\star, \!X'), \texttt{inspect})$.
  \end{enumerate}
\end{enumerate}

\begin{proof}[Proof of \pref{lem:hardcore-good-local-sampler}]
  For \pref{cond:bi-good-local-sampler}, we verify the forward condition below. The reverse direction follows by applying the forward construction to the reversed simulator sequence. Hence the forward verification applies verbatim, with the same parameters.
  \begin{enumerate}
    \item For the local dependency, suppose that the top element of the stack is $((u, t, \!X), \texttt{state})$. The transition rules only query the initial configuration $\sigma_0$ at vertex $u$ and the random sequence $\*\rho$ at time $t$. Hence the local dependency is satisfied with $B = 1$.
    A transition queries the initial configuration only if the vertex has not been updated before.
    Since the update sequence $\*u$ and $\*u^{\leftarrow}$ is cyclic, each backward $\texttt{Last}$ step moves the time back by at most $n$, so no such query occurs along paths with $\ell\le \floor{(t - n) / n}$.
    \item Fix $\sigma_0\in\supp(\mu)$ and $\ell\ge1$.
    The case $1\le\ell<20\Delta$ is trivial, so assume $\ell\ge20\Delta$.
    Process recursive calls in stack order, and call a request fresh if its value is not already recorded in the threaded cache $\!X$ when its \texttt{init} state is reached.
    A fresh child call is a recursive request pushed by the current call that is not a cache hit at its \texttt{init} state.
    Let $Z_k$ be the number of such fresh child calls produced by the $k$-th fresh call.
    Cache hits do not count as fresh child calls and do not read any mark.
    Depth-first execution, cache merging, and strict time decrease imply that each pair $(v,s)$ is processed as fresh at most once.
    Since every positive time updates a unique vertex and the time parameter strictly decreases along recursive edges, a fresh positive-time call exposes a mark coordinate that has not been queried before.
    Thus, conditional on the past, $Z_k$ is stochastically dominated by an independent variable $Y_k$ with $\Pr{Y_k=0}=\frac{1}{1+\lambda}$ and $\Pr{Y_k=\Delta}=\frac{\lambda}{1+\lambda}$.
    Let $(Y_k)_{k\ge1}$ be independent copies of this variable, independent of the actual exploration. Define the comparison process $(\Phi_k)_{k\ge0}$ by $\Phi_0=1$ and
    \begin{equation*}
      \Phi_k-\Phi_{k-1}=Y_k-1,\qquad 0\le Y_k\le\Delta,
    \end{equation*}
    so the true fresh-call queue is dominated by $\Phi$.
    Moreover, $\E{Y_k}\le \Delta\lambda/(1+\lambda)\le 1-\eta/2$, so $\Phi$ has drift at most $-\eta/2$, while one increment has range at most $\Delta$.
    A fresh hard-core call uses at most $20\Delta$ automaton transitions outside its child calls.
    Let $m=\lfloor \ell/(20\Delta)\rfloor$, so $m\ge \ell/(40\Delta)$.
    If $L_{\mathrm{stop}}(\*\rho,\sigma_0)>\ell$, then the true fresh-call queue is still positive after $m$ fresh-call steps; under the above domination coupling, this implies $\Phi_m>0$.
    Hoeffding's inequality for the independent comparison variables gives
    \begin{equation*}
      \Pr[\*\rho\sim \+P_{1:T}]{L_{\mathrm{stop}}(\*\rho,\sigma_0)>\ell}
      \le \Pr[Y_1,\ldots,Y_m]{\sum_{k=1}^{m}(\Phi_k-\Phi_{k-1})\ge0}
      \le \exp\left(-\frac{\eta^2 m}{2\Delta^2}\right)
      \le C\alpha^\ell. \qedhere
    \end{equation*}
  \end{enumerate}
\end{proof}

\subsection{Soft Constraint Graphical Model}

In this section, we consider two-spin systems with soft constraints, including the near-critical Ising model.
The graph is $G=(V,E)$ with maximum degree $\Delta$.
Each vertex $v$ has external field $\lambda_v=(\lambda_v(-1),\lambda_v(+1))$, and each edge $e$ has interaction matrix $A_e=(A_e(\chi_1,\chi_2))_{\chi_1,\chi_2\in\{\pm 1\}}$.
In the sampler description below, we use $\chi$ for a proposed spin and $\sgn$ for a resolved spin value.
The parameters satisfy the following conditions.
\begin{condition}
  \label{cond:soft-parameters}
  There is a constant $\eta\in(0,1)$ such that the parameters satisfy the following:
    \begin{itemize}
    \item \textbf{Normalization:} for every vertex $v$ and every edge $e$,
    \begin{equation*}
      \lambda_v(-1)+\lambda_v(+1)=1, \qquad \max_{\chi_1,\chi_2} A_e(\chi_1,\chi_2)=1.
    \end{equation*}
    \item \textbf{Soft constraint:} for every edge $e$ and spins $\chi_1,\chi_2\in\{\pm 1\}$,
    \begin{equation*}
      A_e(\chi_1,\chi_2)\ge \kappa(\Delta,\eta)\defeq 1-\frac{1-\eta}{2\Delta}.
    \end{equation*}
  \end{itemize}
\end{condition}

Let $\mu$ be the Gibbs distribution on $\{\pm 1\}^V$ with
\begin{equation*}
  \mu(\sigma)\propto
  \prod_{v\in V}\lambda_v(\sigma(v))
  \prod_{\{u,v\}\in E}A_{\{u,v\}}(\sigma(u),\sigma(v)).
\end{equation*}

\begin{lemma}[Bi-directional good local sampler for soft constraints]
  \label{lem:soft-good-local-sampler}
  Fix $\Delta\ge1$ and $\eta\in(0,1)$.
  Let $(G,\mu)$ be a two-spin Gibbs model of maximum degree at most $\Delta$ satisfying \pref{cond:soft-parameters}.
  Then, for every time horizon $T\ge n$, there is an update simulator $((\+P_t,\psi_t))_{t\in[T]}$ with mark domain $[Q]$ such that $(G,\mu)$ satisfies the bi-directional $\goodsampler$-good local sampler condition along $\*u$, where
  $$K=\lceil \log(2\Delta/\eta)/\log(2/(1-\eta))\rceil, \quad Q=12^{K(\Delta+1)},$$
  $$B = 1, \quad C=\exp(\eta^2/(16K^2\Delta^2)),\quad \alpha=\exp(-\eta^2/(640K^3\Delta^3)).$$
\end{lemma}

We first derive the learning theorem and the Ising corollary from the sampler lemma, and then verify the lemma by describing the local sampler and its automaton implementation.

\begin{proof}[Proof of \pref{thm:intro-soft-learning}]
  Normalize each vertex potential by a positive scalar and each edge matrix by its maximum entry; this does not change the Gibbs distribution.
  Under the hypothesis of \pref{thm:intro-soft-learning}, the normalized model satisfies \pref{cond:soft-parameters}.
  By \pref{lem:soft-good-local-sampler}, it satisfies the bi-directional good local sampler condition with constants depending only on $\Delta$ and $\eta$.
  Applying \pref{thm:good-sampler-polynomial-approximation} with accuracy $\varepsilon/2$, followed by the low-degree regression theorem in \pref{sec:low-degree-learning}, gives a learner using $n^{\log^{O(d)}(n/\varepsilon)}$ samples and running time $n^{\log^{O(d)}(n/\varepsilon)}$, where the hidden constant in $O(d)$ depends only on $c,\Delta,\eta$, with error at most $\varepsilon$ and success probability at least $0.9$.
\end{proof}

\begin{proof}[Proof of \pref{cor:intro-ising-learning}]
  For the Ising model, edge normalization makes the smallest entry equal to $\min\{\beta,\beta^{-1}\}$, so the stated range of $\beta$ implies the soft-constraint condition in \pref{thm:intro-soft-learning}.
\end{proof}

We now turn to the proof of \pref{lem:soft-good-local-sampler}.
The construction again has two layers: a recursive rejection resolver for the soft heat-bath update, and a finite automaton that implements this recursion.
Set the attempt cutoff
\begin{equation*}
  K=K(\Delta,\eta)=\left\lceil \frac{\log(2\Delta/\eta)}{\log(2/(1-\eta))}\right\rceil .
\end{equation*}
As before, write $i_t$ for the vertex updated at time $t$, and write $\texttt{Last}(v,t)$ for the last time no later than $t$ at which $v$ is updated, with value $0$ if no such time exists.

We first specify the finite marks that drive the rejection resolver.
These marks are $\*\rho\in[Q]^T$ with product law
$\+P_{1:T}=\bigotimes_{t=1}^T\+P_t$.
We regard $V$ as equipped with an arbitrary fixed order.
Whenever we write $N(u)=\{w_1,\ldots,w_{d_u}\}$, the neighbors are listed in
this order.
For the edge $\{u,w_i\}$, let $\mathcal B_{u,i}$ be the partition of $[0,1]$ induced by $\kappa(\Delta,\eta)$ and the four thresholds $A_{\{u,w_i\}}(\chi_1,\chi_2)$.
The buckets are formal labels for the resulting cells of $[0,1]$, together
with their induced order information: each bucket determines all comparisons
with these thresholds, and we write $b\prec\theta$ when the cell labeled by
$b$ lies below $\theta$.
For the fallback, let $\mathcal C_u$ be the partition induced by the values $\mu_u^\tau(+1)$ over neighbor configurations $\tau\in\{\pm 1\}^{N(u)}$.

For each vertex $u$, the law of the decoded mark is a product law: the proposals $\Xi^a$ have law $\lambda_u$, the rejection buckets $\Lambda_i^a$ have the bucket probabilities induced by $\mathcal B_{u,i}$, and the fallback bucket $\vartheta$ has the bucket probabilities induced by $\mathcal C_u$; all these components are mutually independent.
At each update time $t\ge1$, the mark is $\rho_t$. For the updated vertex $u_t$, write its decoded components as $\decode_{u_t}(\rho_t)=(\Xi_t,\Lambda_t,\vartheta_t)$, where $\Xi_t=(\Xi_t^1,\ldots,\Xi_t^K)$ are the proposed spins, $\Lambda_t=(\Lambda_{t,i}^a)$ stores the test buckets for attempt $a\in[K]$ and neighbor position $i\in[d_{u_t}]$, and $\vartheta_t\in\mathcal C_{u_t}$ is the fallback bucket. There is no mark at time $0$, and no decoding is used for frames with $t=0$.
The common mark alphabet can be chosen as $[Q]$ with, for instance, $Q = 12^{K(\Delta+1)}$.
Indeed, for a vertex $u$ with $d_u\le\Delta$, the decoded local mark has at most
$2^K \cdot 6^{Kd_u} \cdot (2^{d_u}+1)\le 12^{K(\Delta+1)}$ possible values: the factor $2^K$
counts the proposed spins, each rejection bucket has at most six labels, and
the fallback partition has at most $2^{d_u}+1$ cells.
Here $[Q]$ denotes a set of $Q$ formal symbols used to index the common mark
alphabet, rather than the numerical interval $\{1,\ldots,Q\}$ with arithmetic
meaning.
Although the natural alphabet depends on $u$, we inject each local mark space
into this common set of labels and write $\decode_u$ for the corresponding
decoding map, extended arbitrarily outside the image.
Thus comparisons such as $\Lambda_{t,i}^a\prec A_{\{u,w_i\}}(\chi,\sgn)$ are
comparisons between a decoded bucket label and a threshold, not comparisons
between the raw mark symbol $\rho_t$ and a number.
For $1\le t\le T$, the law $\+P_t$ is this finite product law for
the proposals and buckets associated with the updated vertex $u_t$.
Thus $\rho_t\sim\+P_t$, the laws $\+P_t$ may vary with $t$, and
$\decode_{u_t}(\rho_t)$ contains exactly the random choices needed by the local update.

Let $\psi_t$ be the deterministic one-site rule induced by this decoding: it
runs the $K$ rejection tests against the neighbor configuration, then uses
$\vartheta_t$ for the fallback heat-bath update. Thus $((\+P_t,\psi_t))_{t\in[T]}$ is
the forward update simulator.
Indeed, for a fixed neighbor configuration $\tau$, one rejection attempt proposing $\chi$ accepts with probability $\prod_i A_{\{u,w_i\}}(\chi,\tau(w_i))$, so conditional on acceptance the proposal has law $\mu_u^\tau$.
If all attempts fail, the fallback bucket samples from the same law; hence $\Pr[\rho_t\sim\+P_t]{\psi_t(\rho_t,\tau)=c}=\mu_{u_t}^\tau(c)$, as required in \pref{def:update-simulator}.
The reversal of this simulator is
$((\+P_{T-s+1},\psi_{T-s+1}))_{s\in[T]}$ along $\*u^\leftarrow$.

\paragraph{Local sampler.}
We first describe the recursive procedure that these marks drive.
Fix an initial configuration $\sigma_0\in\supp(\mu)$ and a mark sequence $\*\rho\in[Q]^T$.
For a vertex $v$ and a time $0\le t\le T$, the resolver $\resolve(v,t)$ returns the value of $\sigma_t(v)$ in the Glauber trajectory driven by $\*\rho$.
In each rejection attempt, a test bucket below $\kappa(\Delta,\eta)$ certifies acceptance at that neighbor without reading its spin.
Only the remaining tests require recursive calls to resolve the relevant neighbor values.

\begin{algorithm}[htbp]
  \caption{Soft-constraint local resolver $\resolve(v,t)$}
  \label{alg:soft-local-resolver}
  \KwIn{a vertex $v$ and a time $0\le t\le T$}
  \KwOut{the resolved spin $\sigma_t(v)\in\{\pm 1\}$}

  Set $s\gets \texttt{Last}(v,t)$\;
  \If{$s=0$}{
    \Return{$\sigma_0(v)$}\;
  }

  Write $N(v)=\{w_1,\ldots,w_{d_v}\}$ and $\decode_v(\rho_s)=(\Xi_s,\Lambda_s,\vartheta_s)$\;
  \For{$a$ \textnormal{from} $1$ \textnormal{to} $K$}{
    Set $\chi\gets \Xi_s^a$ and $\mathsf{accepted}\gets 1$\;
    \For{$i$ \textnormal{from} $1$ \textnormal{to} $d_v$ \textnormal{while} $\mathsf{accepted}=1$}{
      \If{$\Lambda_{s,i}^a\not\prec\kappa(\Delta,\eta)$}{
        Set $r\gets \texttt{Last}(w_i,s)$ and $\sgn\gets\resolve(w_i,r)$\;
        \If{$\Lambda_{s,i}^a\not\prec A_{\{v,w_i\}}(\chi,\sgn)$}{
          Set $\mathsf{accepted}\gets 0$\;
        }
      }
    }
    \If{$\mathsf{accepted}=1$}{
      \Return{$\chi$}\;
    }
  }

  For every $i\in[d_v]$, set $\tau(w_i)\gets \resolve(w_i,\texttt{Last}(w_i,s))$\;
  \If{$\vartheta_s\prec \mu_v^\tau(+1)$}{
    \Return{$+1$}\;
  }
  \Return{$-1$}\;
\end{algorithm}

Conditioned on accepting in any rejection attempt, the proposed spin has exactly the heat-bath law at $v$.
If all $K$ attempts fail, the fallback step explicitly collects the neighbor configuration and samples from the same conditional distribution.
Hence the resolver returns the desired Glauber update while only recursing on neighbor values that are not certified by the soft lower bound.

\paragraph{Automaton representation.}
Again, we implement the resolver above by an automaton.
For a target vertex $v$ and time $0\le t\le T$, we describe the automaton resolving $\sigma_t(v)$.
Each stack element has the form $((u,t,\!X),\texttt{state})$.
The frame $(u,t,\!X)$ asks for $\sigma_t(u)$, and $\!X$ stores resolved values as triples $(v,s,\sgn)$, meaning $\sigma_s(v)=\sgn$.
Here $s,t$ are times, $u,w$ are vertices, $a,i$ are indices, $\mathsf{cnt}$ is the number of failed rejection attempts, $K$ is the attempt cutoff, and spin values are always in $\{\pm 1\}$.
The cache $\!X$ is threaded through the depth-first execution: a child frame is initialized with the current cache of its parent, and when the child returns, the child's final cache is merged back into the parent.
Thus $\!X$ records all resolved triples known to the current branch of the execution.
When a rule creates a parent element and a child element, the parent stays below the child, and the child is the new top of the stack.

We categorize the possible stack elements by the execution step of the current call:
\begin{itemize}
  \item \textbf{Init} $((u,t,\!X),\texttt{init})$: the call has just started resolving $\sigma_t(u)$. As in the hard-core automaton, we use the convention that $t=0$ or $u_t=u$; otherwise replace the requested time by $\texttt{Last}(u,t)$.
  \item \textbf{Propose} $((u,t,\!X),(\texttt{propose},\mathsf{cnt}))$: the call is in the rejection-sampling phase, after $\mathsf{cnt}$ failed attempts, and is ready either to start the next attempt or to enter the fallback step if $\mathsf{cnt}=K$.
  \item \textbf{Test} $((u,t,\!X),(\texttt{test},\mathsf{cnt},w,\chi))$: attempt $\mathsf{cnt}+1$ has proposed spin $\chi$, and the sampler is checking the current neighbor $w$ in the fixed order on $N(u)$.
  \item \textbf{Hold\_Test} $((u,t,\!X),(\texttt{hold\_test},\mathsf{cnt},w,s,\chi))$: the current test for $w$ cannot be decided from the bucket alone, so the call is paused while a child call resolves $\sigma_s(w)$.
  \item \textbf{Collect} $((u,t,\!X),\texttt{collect})$: all rejection attempts have failed, and the call is collecting the neighbor spins needed for the final conditional update.
  \item \textbf{Hold\_Collect} $((u,t,\!X),(\texttt{hold\_collect},w,s))$: the call is in the fallback conditional-sampling step and is collecting the neighbor configuration. The value $\sigma_s(w)$ is not yet in the current frame's cache, so the call is paused while a child call resolves it.
  \item \textbf{Halt} $((u,t,\!X),(\texttt{halt},\sgn))$: the call has finished resolving $\sigma_t(u)$ and is ready to return spin $\sgn$ to its parent, or to output $\sgn$ if the stack has no parent.
\end{itemize}

Initially, $S$ contains $((v,\texttt{Last}(v,t),\emptyset),\texttt{init})$.
In each step, let $((u,t,\!X),\texttt{state})$ be the top element.
When a transition needs the mark at a positive time $t$, write $\decode_u(\rho_t)=(\Xi_t,\Lambda_t,\vartheta_t)$.
The automaton then updates the top stack element according to the following rules:
\begin{enumerate}
  \item If $\texttt{state}=\texttt{init}$:
  \begin{enumerate}
    \item If $(u,t,\sgn)\in \!X$ for some $\sgn$, replace the top element by $((u,t,\!X),(\texttt{halt},\sgn))$.
    \item If $t=0$ and no such cached triple exists, query $\sigma_0(u)$ without decoding any mark, and replace the top element by $((u,t,\!X),(\texttt{halt},\sigma_0(u)))$.
    \item If $t>0$ and no such cached triple exists, replace the top element by $((u,t,\!X),(\texttt{propose},0))$.
  \end{enumerate}

  \item If $\texttt{state}=(\texttt{propose},\mathsf{cnt})$:
  \begin{enumerate}
    \item If $\mathsf{cnt}<K$, start the next attempt: set $\chi=\Xi_t^{\mathsf{cnt}+1}$ from $\decode_u(\rho_t)$.
    If $d_u=0$, replace the top element by $((u,t,\!X),(\texttt{halt},\chi))$; otherwise replace it by the element with frame $(u,t,\!X)$ and state $(\texttt{test},\mathsf{cnt},w_1,\chi)$.
    \item If $\mathsf{cnt}=K$, replace the top element by $((u,t,\!X),\texttt{collect})$.
  \end{enumerate}

  \item If $\texttt{state}=(\texttt{test},\mathsf{cnt},w_i,\chi)$ for some $i\in[d_u]$, let $a=\mathsf{cnt}+1$.
  When the test at $w_i$ passes, keep the frame $(u,t,\!X)$ and set the state to $(\texttt{test},\mathsf{cnt},w_{i+1},\chi)$ if $i<d_u$, and to $(\texttt{halt},\chi)$ otherwise.
  \begin{enumerate}
    \item Inspect $\Lambda_{t,i}^a$ from $\decode_u(\rho_t)$.
    If $\Lambda_{t,i}^a\prec \kappa(\Delta,\eta)$, apply this passed-test transition.
    \item If $\Lambda_{t,i}^a\not\prec \kappa(\Delta,\eta)$, set $s=\texttt{Last}(w_i,t)$.
    If no triple $(w_i,s,\cdot)$ lies in $\!X$, replace the top by the parent element $((u,t,\!X),(\texttt{hold\_test},\mathsf{cnt},w_i,s,\chi))$ and push the child element $((w_i,s,\!X),\texttt{init})$ above it.
    Otherwise, let $(w_i,s,\sgn)\in\!X$ be the cached value.
    When $\Lambda_{t,i}^a\prec A_{\{u,w_i\}}(\chi,\sgn)$, apply the passed-test transition; otherwise replace the top by the element with frame $(u,t,\!X)$ and state $(\texttt{propose},\mathsf{cnt}+1)$.
  \end{enumerate}

  \item If $\texttt{state}=\texttt{collect}$, let
  $\mathcal R_{\!X}=\set{w\in N(u): (w,\texttt{Last}(w,t),\sgn)\in\!X\text{ for some }\sgn}$
  be the set of neighbors whose values at time $t$ have been resolved.
  \begin{enumerate}
    \item If $N(u)\subseteq\mathcal R_{\!X}$, the neighbor configuration $\tau_{\!X}$ is determined by $\tau_{\!X}(w)=\sgn$ whenever $(w,\texttt{Last}(w,t),\sgn)\in\!X$.
    Read $\vartheta_t$ from $\decode_u(\rho_t)$. Keeping the frame $(u,t,\!X)$ unchanged, set the state to $(\texttt{halt},+1)$ if $\vartheta_t\prec \mu_u^{\tau_{\!X}}(+1)$, and to $(\texttt{halt},-1)$ otherwise.
    \item If $N(u)\nsubseteq\mathcal R_{\!X}$, pick the smallest vertex $w\in N(u)\setminus\mathcal R_{\!X}$, set $s=\texttt{Last}(w,t)$, replace the top by the parent element with frame $(u,t,\!X)$ and state $(\texttt{hold\_collect},w,s)$, and push the child element $((w,s,\!X),\texttt{init})$ above it.
  \end{enumerate}

  \item If $\texttt{state}=(\texttt{halt},\sgn)$, write the top element as $((u^\circ,t^\circ,\!X^\circ),(\texttt{halt},\sgn))$ and pop it:
  \begin{enumerate}
    \item If the stack is empty, return $\sgn$ and enter $\bot_\sgn$.
    \item Otherwise write the new top frame as $(u^\star,t^\star,\!X^\star)$ and handle its state as follows.
    \begin{enumerate}
      \item If the new top state is $(\texttt{hold\_collect},w,s)$, set $\!X'=\!X^\star\cup\!X^\circ\cup\{(w,s,\sgn)\}$ and replace the top element by $((u^\star,t^\star,\!X'),\texttt{collect})$.
      \item If the new top state is $\texttt{hold\_test}$ with data $(\mathsf{cnt},w,s,\chi)$, we set $\!X'=\!X^\star\cup\!X^\circ\cup\{(w,s,\sgn)\}$.
      Decode $\rho_{t^\star}$ as $(\Xi,\Lambda,\vartheta)$.
      Write $N(u^\star)=\{z_1,\ldots,z_{d_{u^\star}}\}$, let $z_i=w$, and set $a=c+1$.
      If $\Lambda_i^a\prec A_{\{u^\star,w\}}(\chi,\sgn)$ and $i<d_{u^\star}$, replace the top element by the element with frame $(u^\star,t^\star,\!X')$ and state $(\texttt{test},\mathsf{cnt},z_{i+1},\chi)$.
      If the same comparison holds and $i=d_{u^\star}$, replace the top element by $((u^\star,t^\star,\!X'),(\texttt{halt},\chi))$.
      If the comparison fails, replace the top element by $((u^\star,t^\star,\!X'),(\texttt{propose},\mathsf{cnt}+1))$.
    \end{enumerate}
  \end{enumerate}
\end{enumerate}

\begin{proof}[Proof of \pref{lem:soft-good-local-sampler}]
  Similar to the hard-core case, it suffices to verify the forward condition with respect to $\*u$.
  The transition rules of the automaton implement the local sampler above.
  A successful rejection attempt accepts a proposal with the correct heat-bath weight, while the fallback step collects the needed neighbor values and samples the exact conditional distribution.
  Hence the output is the desired value $\sigma_t(v)$.
  \begin{enumerate}

    \item The parameter $B$ counts input coordinates read in one transition.
    Even in the fallback case, the collect phase, represented by $\texttt{collect}$ and $(\texttt{hold\_collect},w,s)$, resolves at most one missing neighbor at a time, so a transition reads at most one initial spin or one mark $\rho_t$.
    Cached triples in $\!X$ are state information and do not count as new input queries.
    Hence $B=1$.
    The initial configuration is queried only when the queried vertex has not been updated.
    Since $\*u$ and $\*u^{\leftarrow}$ are cyclic, each backward $\texttt{Last}$ step moves the time back by at most $n$, so no such query occurs along paths $(x_0,\ldots,x_\ell)$ with $\ell\le \floor{(t - n) / n}$.

    \item Fix $\sigma_0\in\supp(\mu)$ and $\ell\ge1$.
    Let $C=\exp(\eta^2/(16K^2\Delta^2))$ and $\alpha=\exp(-\eta^2/(640K^3\Delta^3))$.
    The case $1\le\ell<40K\Delta$ is trivial since $C\alpha^\ell\ge1$, so assume $\ell\ge40K\Delta$.
    Process recursive calls in stack order.
    A request is called fresh if its value is not already recorded in the threaded cache $\!X$ when its \texttt{init} state is reached.
    Let $Z_k$ be the number of recursive requests pushed while processing the $k$-th fresh request that are themselves fresh when their \texttt{init} states are reached; if fewer than $k$ fresh requests occur, set $Z_k=0$.
    Thus $Z_k$ counts only genuinely new child calls of the $k$-th fresh parent call.
    Cache hits are not included in $Z_k$ and do not read any mark.
    Depth-first execution, cache merging, and strict time decrease imply that each pair $(v,s)$ is processed as fresh at most once.
    Since every positive time updates a unique vertex and recursive calls strictly decrease the time parameter, a fresh positive-time call exposes a mark coordinate that has not been queried before.
    Let $\mathcal F_{k-1}$ be the history after the first $k-1$ fresh requests have been processed, and define $(\Phi_k)_{k\ge0}$ by
    \begin{equation*}
      \Phi_0=1,\qquad \Phi_k-\Phi_{k-1}=Z_k-1 .
    \end{equation*}
    Up to extinction, $\Phi_k$ is exactly the size of the fresh-call queue after $k$ fresh requests have been processed; after extinction the convention $Z_k=0$ only makes $\Phi$ continue to decrease.

    We bound the conditional expectation of $Z_k$. Set $q=\Delta(1-\kappa)=(1-\eta)/2$.
    If the $k$-th fresh request does not exist, or if it has time $0$, then $Z_k=0$.
    Otherwise fix this fresh request, say it resolves $\sigma_t(u)$ with $t>0$.
    Condition on $\mathcal F_{k-1}$, on a given rejection attempt being reached, and on its proposal $\chi$.
    The neighbor values tested in this attempt depend only on $\sigma_0$ and on marks at times smaller than $t$, while the buckets of the present attempt are still unused.
    Hence, for each neighbor, the probability that its value is queried is at most $1-\kappa$; summing over at most $\Delta$ neighbors gives conditional expected fresh children at most $q$.
    After also fixing the true neighbor values, the attempt fails only if for some neighbor $w$ its bucket is not below $A_{\{u,w\}}(\chi,\sigma_{\texttt{Last}(w,t)}(w))$.
    This threshold is at least $\kappa$, so the conditional failure probability is at most $\Delta(1-\kappa)=q$ by the union bound.
    Writing $A_a$ for the event that attempt $a$ is reached, we have $\Pr{A_a\given\mathcal F_{k-1}}\le q^{a-1}$, and the rejection phase contributes at most
    \begin{equation*}
      \sum_{a=1}^{K}
      \Pr{A_a\given\mathcal F_{k-1}}
      \cdot
      \oE\left[\#\text{ fresh children in attempt }a\given\mathcal F_{k-1},A_a\right] \le
      \sum_{a=1}^{K}q^{a-1}q
      \le
      \frac{q}{1-q}
      =
      \frac{1-\eta}{1+\eta}.
    \end{equation*}
    The fallback phase is reached only if all $K$ rejection attempts fail, an event of conditional probability at most $q^K$ by the same recursion.
    Since the fallback can request at most $\Delta$ neighbor values, its expected contribution is at most $\Delta q^K\le\eta/2$ by the choice of $K$.
    Combining the rejection and fallback contributions,
    \begin{equation*}
      \oE[Z_k\given\mathcal F_{k-1}]
      \le
      \frac{1-\eta}{1+\eta}+\frac{\eta}{2}
      \le 1-\frac{\eta}{2}.
    \end{equation*}
    Thus, for every $k$, $\oE\left[\Phi_k-\Phi_{k-1}\given\mathcal F_{k-1}\right] = \oE[Z_k-1\given\mathcal F_{k-1}] \le -\eta/2$.
    Each increment has range at most $2K\Delta$, since a fresh call can produce at most $K\Delta$ children during the rejection attempts and at most $\Delta$ more in the fallback phase.
    A fresh soft call uses at most $40K\Delta$ automaton transitions outside its child calls.
    Let $m=\lfloor \ell/(40K\Delta)\rfloor$ and since $40K\Delta m\le \ell$, the event $L_{\mathrm{stop}}(\*\rho,\sigma_0)>\ell$ implies that the fresh-call queue is still nonempty after $m$ fresh requests have been processed; otherwise the automaton would stop within $\ell$ transitions.
    Up to extinction $\Phi$ is exactly this queue size, so the tail event is contained in $\Phi_m\ge\Phi_0$, equivalently in $\sum_{k=1}^{m}(\Phi_k-\Phi_{k-1})\ge0$.
    Hoeffding--Azuma's inequality gives
    \begin{equation*}
      \Pr[\*\rho\sim \+P_{1:T}]{L_{\mathrm{stop}}(\*\rho,\sigma_0)>\ell}
      \le \Pr[\*\rho\sim \+P_{1:T}]{\sum_{k=1}^{m}(\Phi_k-\Phi_{k-1})\ge0}
      \le
      \exp\left(-\frac{\eta^2m}{8K^2\Delta^2}\right)
      \le
      C\alpha^\ell. \qedhere
    \end{equation*}
  \end{enumerate}
\end{proof}

\ifthenelse{\boolean{doubleblind}}
{
    \section*{AI Disclaimer}
The proof idea in \Cref{sec:low-degree-product} was suggested by ChatGPT 5.5 Pro. All other proof ideas are due to the human authors. ChatGPT 5.5 Pro and Codex were used to assist with proof details, exposition, and verification. The authors take full responsibility for the correctness of all proofs.
}
{
    \section*{Acknowledgments} We thank Xiaoyu Chen, Yangjing Dong, Fengning Ou, Chunyang Wang, Xinyuan Zhang and Mingnan Zhao for helpful discussions. Weiming Feng acknowledges the support of ECS grant 27202725 from Hong Kong RGC. Yixiao Yu acknowledges the support of the National Natural Science Foundation of China under Grant No. 62472212.

    The proof idea in \Cref{sec:low-degree-product} was suggested by ChatGPT 5.5 Pro. All other proof ideas are due to the human authors. ChatGPT 5.5 Pro and Codex were used to assist with proof details, exposition, and verification. The authors take full responsibility for the correctness of all proofs.
}
\begingroup
\setstretch{1.18}
\printbibliography

@InProceedings{DongM25,
  author    = {Dingding Dong and Nitya Mani},
  booktitle = {{SAT}},
  title     = {Random Local Access for Sampling k-SAT Solutions},
  year      = {2025},
  pages     = {13:1--13:24},
  volume    = {341},
}

@inproceedings{Weitz06,
  author       = {Dror Weitz},
  title        = {Counting independent sets up to the tree threshold},
  booktitle    = {STOC},
  pages        = {140--149},
  year         = {2006},
}

@inproceedings{BiswasRY20,
  author       = {Amartya Shankha Biswas and
                  Ronitt Rubinfeld and
                  Anak Yodpinyanee},
  title        = {Local Access to Huge Random Objects Through Partial Sampling},
  booktitle    = {ITCS},
  volume       = {151},
  pages        = {27:1--27:65},
  year         = {2020},
}

@article{CCLZ26Subquadratic,
  author        = {Xiaoyu Chen and Zongchen Chen and Kuikui Liu and Xinyuan Zhang},
  title         = {Subquadratic Counting via Perfect Marginal Sampling},
  year          = {2026},
  archiveprefix = {arXiv},
  eprint        = {2604.02235}
}

@article{CGMV26Learning,
  author        = {Gautam Chandrasekaran and Jason Gaitonde and Ankur Moitra and Arsen Vasilyan},
  title         = {Learning {$\mathsf{AC}^0$} Under Graphical Models},
  year          = {2026},
  archiveprefix = {arXiv},
  eprint        = {2604.06109}
}

@InProceedings{Fel12Learning,
  author    = {Feldman, Vitaly},
  booktitle = {COLT},
  title     = {Learning {DNF} Expressions from {Fourier} Spectrum},
  year      = {2012},
  pages     = {17.1--17.19},
  volume    = {23},
}

@article{FGW25Derandomizing,
  author  = {Feng, Weiming and Guo, Heng and Wang, Chunyang and Wang, Jiaheng and Yin, Yitong},
  journal = {SIAM J. Comput.},
  title   = {Toward Derandomizing {Markov} Chain {Monte Carlo}},
  year    = {2025},
  number  = {3},
  pages   = {775--813},
  volume  = {54},
  doi     = {10.1137/24m1663806}
}

@Article{GJN16Influence,
  author  = {Grimmett, Geoffrey R. and Janson, Svante and Norris, James R.},
  journal = {Adv. Appl. Probab.},
  title   = {Influence in Product Spaces},
  year    = {2016},
  number  = {A},
  pages   = {145--152},
  volume  = {48},
  doi     = {10.1017/apr.2016.46},
}

@Article{Kel11Influences,
  author  = {Keller, Nathan},
  journal = {Comb. Probab. Comput.},
  title   = {On the Influences of Variables on Boolean Functions in Product Spaces},
  year    = {2011},
  number  = {1},
  pages   = {83--102},
  volume  = {20},
  doi     = {10.1017/S0963548310000234},
}

@Article{KalaiKKMS08Agnostically,
  author  = {Kalai, Adam Tauman and Klivans, Adam R. and Mansour, Yishay and Servedio, Rocco A.},
  journal = {SIAM J. Comput.},
  title   = {Agnostically Learning Halfspaces},
  year    = {2008},
  number  = {6},
  pages   = {1777--1805},
  volume  = {37},
  doi     = {10.1137/060649057},
}

@Article{KlivansOS04Learning,
  author  = {Klivans, Adam R. and O'Donnell, Ryan and Servedio, Rocco A.},
  journal = {J. Comput. Syst. Sci.},
  title   = {Learning Intersections and Thresholds of Halfspaces},
  year    = {2004},
  number  = {4},
  pages   = {808--840},
  volume  = {68},
  doi     = {10.1016/j.jcss.2003.11.002},
}

@InProceedings{KanadeM15MCMC,
  author    = {Kanade, Varun and Mossel, Elchanan},
  booktitle = {COLT},
  title     = {{MCMC} Learning},
  year      = {2015},
  pages     = {1101--1128},
  volume    = {40},
}

@Article{KM93LearningDecisionTrees,
  author  = {Kushilevitz, Eyal and Mansour, Yishay},
  journal = {SIAM J. Comput.},
  title   = {Learning Decision Trees Using the {Fourier} Spectrum},
  year    = {1993},
  number  = {6},
  pages   = {1331--1348},
  volume  = {22},
  doi     = {10.1137/0222080},
}

@Article{LMN93Constant,
  author  = {Linial, Nathan and Mansour, Yishay and Nisan, Noam},
  journal = {J. ACM},
  title   = {Constant Depth Circuits, {Fourier} Transform, and Learnability},
  year    = {1993},
  number  = {3},
  pages   = {607--620},
  volume  = {40},
  doi     = {10.1145/174130.174138},
}

@Article{Jackson97DNF,
  author  = {Jackson, Jeffrey C.},
  journal = {J. Comput. Syst. Sci.},
  title   = {An Efficient Membership-Query Algorithm for Learning {DNF} with Respect to the Uniform Distribution},
  year    = {1997},
  number  = {3},
  pages   = {414--440},
  volume  = {55},
  doi     = {10.1006/jcss.1997.1533},
}

@inproceedings{LWY26Local,
  author    = {Hongyang Liu and Chunyang Wang and Yitong Yin},
  booktitle = {SODA},
  title     = {Local {Gibbs} Sampling Beyond Local Uniformity},
  year      = {2026},
  pages     = {996--1025},
  doi       = {10.1137/1.9781611978971.41}
}

@Article{MOO10Noise,
  author  = {Mossel, Elchanan and O'Donnell, Ryan and Oleszkiewicz, Krzysztof},
  journal = {Ann. Math.},
  title   = {Noise Stability of Functions with Low Influences: Invariance and Optimality},
  year    = {2010},
  number  = {1},
  pages   = {295--341},
  volume  = {171},
  doi     = {10.4007/annals.2010.171.295},
}

@book{OD14Analysis,
  author    = {O'Donnell, Ryan},
  publisher = {Cambridge University Press},
  title     = {Analysis of Boolean Functions},
  year      = {2014},
  doi       = {10.1017/CBO9781139814782}
}

@InProceedings{Tal17Tight,
  author    = {Tal, Avishay},
  booktitle = {CCC},
  title     = {Tight Bounds on the {Fourier} Spectrum of {$\mathrm{AC}^0$}},
  year      = {2017},
  pages     = {15:1--15:31},
  doi       = {10.4230/LIPIcs.CCC.2017.15},
}

@InProceedings{AFG25Sink,
  author    = {Anand, Konrad and Freifeld, Graham and Guo, Heng and Wang, Chunyang and Wang, Jiaheng},
  booktitle = {APPROX/RANDOM},
  title     = {Sink-Free Orientations: A Local Sampler with Applications},
  year      = {2025},
  pages     = {60:1--60:19},
  doi       = {10.4230/LIPIcs.APPROX/RANDOM.2025.60},
}

@InProceedings{AFFGW24Approximate,
  author    = {Anand, Konrad and Feng, Weiming and Freifeld, Graham and Guo, Heng and Wang, Jiaheng},
  booktitle = {ICALP},
  title     = {Approximate Counting for Spin Systems in Sub-Quadratic Time},
  year      = {2024},
  pages     = {11:1--11:20},
  volume    = {297},
  doi       = {10.4230/LIPIcs.ICALP.2024.11},
}

@Article{AJ22Perfect,
  author  = {Anand, Konrad and Jerrum, Mark},
  journal = {SIAM J. Comput.},
  title   = {Perfect Sampling in Infinite Spin Systems via Strong Spatial Mixing},
  year    = {2022},
  number  = {4},
  pages   = {1280--1295},
  volume  = {51},
  doi     = {10.1137/21m1437433},
}

@InProceedings{Bre15Efficiently,
  author    = {Bresler, Guy},
  booktitle = {STOC},
  title     = {Efficiently Learning {Ising} Models on Arbitrary Graphs},
  year      = {2015},
  pages     = {771--782},
  doi       = {10.1145/2746539.2746631},
}

@InProceedings{DKY26Estimating,
  author    = {Daskalakis, Constantinos and Kandiros, Vardis and Yao, Rui},
  booktitle = {COLT},
  title     = {Estimating {Ising} Models in Total Variation Distance},
  year      = {2026},
}

@InProceedings{HKM17Information,
  author    = {Hamilton, Linus and Koehler, Frederic and Moitra, Ankur},
  booktitle = {NeurIPS},
  title     = {Information Theoretic Properties of {Markov} Random Fields, and their Algorithmic Applications},
  year      = {2017},
  pages     = {2463--2472},
  volume    = {30},
}

@InProceedings{KM17Learning,
  author    = {Klivans, Adam R. and Meka, Raghu},
  booktitle = {FOCS},
  title     = {Learning Graphical Models Using Multiplicative Weights},
  year      = {2017},
  pages     = {343--354},
  doi       = {10.1109/FOCS.2017.39},
}

@InProceedings{VMLC16Interaction,
  author    = {Vuffray, Marc and Misra, Sidhant and Lokhov, Andrey Y. and Chertkov, Michael},
  booktitle = {NeurIPS},
  title     = {Interaction Screening: Efficient and Sample-Optimal Learning of {Ising} Models},
  year      = {2016},
  volume    = {29},
}

@InProceedings{WSD19Sparse,
  author    = {Wu, Shanshan and Sanghavi, Sujay and Dimakis, Alexandros G.},
  booktitle = {NeurIPS},
  title     = {Sparse Logistic Regression Learns All Discrete Pairwise Graphical Models},
  year      = {2019},
  pages     = {8071--8081},
  volume    = {32},
}

@article{ProppW96,
  title = {Exact Sampling with Coupled {{Markov}} Chains and Applications to Statistical Mechanics},
  author = {Propp, James Gary and Wilson, David Bruce},
  year = 1996,
  journal = {Random Struct. Algorithms},
  volume = {9},
  number = {1-2},
  pages = {223--252},
  issn = {1042-9832}
}

@InProceedings{FJS91Improved,
  author    = {Merrick L. Furst and Jeffrey C. Jackson and Sean W. Smith},
  booktitle = {COLT},
  title     = {Improved Learning of AC\({}^{\mbox{0}}\) Functions},
  year      = {1991},
  pages     = {317--325},
}

@Article{BOW10Polynomial,
  author  = {Eric Blais and Ryan O'Donnell and Karl Wimmer},
  journal = {Mach. Learn.},
  title   = {Polynomial regression under arbitrary product distributions},
  year    = {2010},
  number  = {2-3},
  pages   = {273--294},
  volume  = {80},
  doi     = {10.1007/S10994-010-5179-6},
}

@InProceedings{GKK08Agnostically,
  author    = {Gopalan, Parikshit and Kalai, Adam Tauman and Klivans, Adam R.},
  booktitle = {STOC},
  title     = {Agnostically learning decision trees},
  year      = {2008},
  pages     = {527--536},
  doi       = {10.1145/1374376.1374451},
}

@Article{Bop97Average,
  author  = {Ravi B. Boppana},
  journal = {Inf. Process. Lett.},
  title   = {The Average Sensitivity of Bounded-Depth Circuits},
  year    = {1997},
  number  = {5},
  pages   = {257--261},
  volume  = {63},
  doi     = {10.1016/S0020-0190(97)00131-2},
}

@Article{Haas01Slight,
  author  = {Håstad, Johan},
  journal = {J. Comput. Syst. Sci.},
  title   = {A Slight Sharpening of LMN},
  year    = {2001},
  number  = {3},
  pages   = {498--508},
  volume  = {63},
  doi     = {10.1006/jcss.2001.1803},
}

@InProceedings{KST09Learning,
  author    = {Adam Tauman Kalai and Alex Samorodnitsky and Shang{-}Hua Teng},
  booktitle = {FOCS},
  title     = {Learning and Smoothed Analysis},
  year      = {2009},
  pages     = {395--404},
  doi       = {10.1109/FOCS.2009.60},
}

@InProceedings{BDM20ID3,
  author    = {Alon Brutzkus and Amit Daniely and Eran Malach},
  booktitle = {COLT},
  title     = {{ID3} Learns Juntas for Smoothed Product Distributions},
  year      = {2020},
  pages     = {902--915},
  volume    = {125},
}

@InProceedings{CKK24Smoothed,
  author    = {Gautam Chandrasekaran and Adam R. Klivans and Vasilis Kontonis and Raghu Meka and Konstantinos Stavropoulos},
  booktitle = {COLT},
  title     = {Smoothed Analysis for Learning Concepts with Low Intrinsic Dimension},
  year      = {2024},
  pages     = {876--922},
  volume    = {247},
}

@InProceedings{CK25Learning,
  author    = {Gautam Chandrasekaran and Adam R. Klivans},
  booktitle = {NeurIPS},
  title     = {Learning Juntas under Markov Random Fields},
  year      = {2025},
}

@InProceedings{GMM25Bypassing,
  author    = {Jason Gaitonde and Ankur Moitra and Elchanan Mossel},
  booktitle = {STOC},
  title     = {Bypassing the Noisy Parity Barrier: Learning Higher-Order Markov Random Fields from Dynamics},
  year      = {2025},
  pages     = {348--359},
  doi       = {10.1145/3717823.3718231},
}

@InProceedings{GMM26Learning,
  author    = {Jason Gaitonde and Ankur Moitra and Elchanan Mossel},
  booktitle = {COLT},
  title     = {Learning Ising Models from Evolutions},
  year      = {2026},
}

@InProceedings{CK25Learninga,
  author    = {Chandrasekaran, Gautam and Klivans, Adam R},
  booktitle = {STOC},
  title     = {Learning the sherrington-kirkpatrick model even at low temperature},
  year      = {2025},
  pages     = {1774--1784},
}

@InProceedings{HWY23Deterministic,
  author    = {He, Kun and Wang, Chunyang and Yin, Yitong},
  booktitle = {SODA},
  title     = {Deterministic counting Lovász local lemma beyond linear programming},
  year      = {2023},
  pages     = {3388--3425},
  doi       = {10.1137/1.9781611977554.ch130},
}

@InProceedings{AGPP23Perfect,
  author    = {Anand, Konrad and Göbel, Andreas and Pappik, Marcus and Perkins, Will},
  booktitle = {APPROX/RANDOM},
  title     = {Perfect Sampling for Hard Spheres from Strong Spatial Mixing},
  year      = {2023},
  pages     = {38:1--38:18},
  volume    = {275},
  doi       = {10.4230/LIPICS.APPROX/RANDOM.2023.38},
}
\endgroup
\appendix

\end{document}